\documentclass{article}

\usepackage{hyperref}       
\usepackage{url}            
\usepackage{booktabs}       
\usepackage{amsfonts,amsmath,amssymb,amsthm,bm}       
\usepackage{nicefrac}       
\usepackage{microtype}      
\usepackage{lipsum}
\usepackage{graphicx}
\usepackage{subcaption}
\usepackage{float}
\usepackage{algorithmic}
\usepackage{ragged2e}
\usepackage{multirow}
\usepackage{tikz}
\usetikzlibrary{bayesnet}
\usetikzlibrary{positioning}
\usepackage{enumerate}

\newtheorem{theorem}{Theorem}[section]
\newtheorem{theorem*}{Theorem}

\newtheorem{corollary*}{Corollary}

\newtheorem{proposition*}{Proposition}
\newtheorem{lemma}{Lemma}[section]
\newtheorem{lemma*}{Lemma}

\newtheorem{definition}{Definition}[section]
\newtheorem{definition*}{Definition}

\usepackage{float}
\floatstyle{ruled}
\newfloat{program}{thp}{lop}
\floatname{program}{Program}

\newcommand\reals{\mathbb{R}} 
\DeclareMathOperator*{\argmin}{argmin} 

\DeclareMathOperator*{\tr}{tr}
\DeclareMathOperator{\vecto}{vec}
\newcommand*{\R}{\mathbb R}

\newcommand*{\cond}{\;\ifnum\currentgrouptype=16 \middle\fi|\;}

\newcommand*{\ttilde}{{\raise.17ex\hbox{$\scriptstyle\sim$}}}
\newcommand*{\tensor}[1]{\bm{\mathcal{#1}}}
\newcommand*{\mat}[1]{\mathbf{#1}}

\newcommand{\diag}{\mathop{\text{diag}}}

\makeatletter
\newsavebox{\mybox}\newsavebox{\mysim}
\newcommand*{\distas}[1]{%
  \savebox{\mybox}{\hbox{\kern3pt$\scriptstyle#1$\kern3pt}}%
  \savebox{\mysim}{\hbox{$\sim$}}%
  \mathbin{\overset{#1}{\kern\z@\resizebox{\wd\mybox}{\ht\mysim}{$\sim$}}}%
}
\makeatother

\makeatletter
\def\moverlay{\mathpalette\mov@rlay}
\def\mov@rlay#1#2{\leavevmode\vtop{%
   \baselineskip\z@skip \lineskiplimit-\maxdimen
   \ialign{\hfil$\m@th#1##$\hfil\cr#2\crcr}}}
\newcommand*{\charfusion}[3][\mathord]{
  #1{\ifx#1\mathop\vphantom{#2}\fi\mathpalette\mov@rlay{#2\cr#3}}
  \ifx#1\mathop\expandafter\displaylimits\fi}
\makeatother

\usepackage{pdflscape}
\usepackage{afterpage}
\newtheoremstyle{algodesc}{}{}{}{}{\bfseries}{.}{ }{}%
\theoremstyle{algodesc}

\usepackage{microtype}
\usepackage{graphicx}
\usepackage{booktabs} 

\usepackage{hyperref}

\usepackage[accepted]{icml2021}

\icmltitlerunning{SG-PALM: a Fast Physically Interpretable Tensor Graphical Model}

\begin{document}

\twocolumn[
\icmltitle{SG-PALM: a Fast Physically Interpretable Tensor Graphical Model
}



\icmlsetsymbol{equal}{*}

\begin{icmlauthorlist}
\icmlauthor{Yu Wang}{umich}
\icmlauthor{Alfred Hero}{umich}
\end{icmlauthorlist}

\icmlaffiliation{umich}{University of Michigan, Ann Arbor, Michigan, USA}

\icmlcorrespondingauthor{Yu Wang}{wayneyw@umich.edu}

\icmlkeywords{Machine Learning, ICML}

\vskip 0.3in
]



\printAffiliationsAndNotice{}  

\begin{abstract}

We propose a new graphical model inference procedure, called SG-PALM, for learning conditional dependency structure of high-dimensional tensor-variate data. Unlike most other tensor graphical models the proposed model is interpretable and computationally scalable to high dimension. Physical interpretability follows from the Sylvester generative (SG) model on which SG-PALM is based: the model is exact for any observation process that is a solution of a partial differential equation of Poisson type. Scalability follows from the fast proximal alternating linearized minimization (PALM) procedure that SG-PALM uses during training. We establish that SG-PALM converges linearly (i.e., geometric convergence rate) to a global optimum of its objective function. We demonstrate the scalability and accuracy of SG-PALM for an important but challenging climate prediction problem: spatio-temporal forecasting of solar flares from multimodal imaging data.
\end{abstract}

\section{Introduction}\label{sec:intro}
High-dimensional tensor-variate data arise in computer vision (video data containing multiple frames of color images), neuroscience (EEG measurements taken from different sensors over time under various experimental conditions), and recommending system (user preferences over time). Due to the non-homogeneous nature of these data, second-order information that encodes (conditional) dependency structure within the data is of interest. Assuming the data are drawn from a tensor normal distribution, a straightforward way to estimate this structure is to vectorize the tensor and estimate the underlying Gaussian graphical model associated with the vector. However, such an approach ignores the tensor structure and requires estimating a rather high dimensional precision matrix, often with insufficient sample size. For instance, in the aforementioned EEG application the sample size is one if we aim to estimate the dependency structure across different sensors, time and experimental conditions for a single subject. To address such sample complexity challenges, sparsity is often imposed on the covariance $\mat{\Sigma}$ or the inverse covariance $\mat{\Omega}$, e.g., by using a sparse Kronecker product (KP) or Kronecker sum (KS) decomposition of $\mat{\Sigma}$ or $\mat{\Omega}$. The earliest and most popular form of sparse structured precision matrix estimation approaches represent $\mat{\Omega}$, equivalently $\mat\Sigma$, as the KP of smaller precision/covariance matrices~\citep{allen2010transposable,leng2012sparse,yin2012model,tsiligkaridis2013convergence,zhou14,lyu2019tensor}. The KP structure induces a generative representation for the tensor-variate data via a separable covariance/inverse covariance model. Alternatively, \citet{kalaitzis2013bigraphical,greenewald2019tensor} proposed to model inverse covariance matrices using a KS representation. \citet{rudelsonzhou17errinvardependent,parketal17_kroneckersum} proposed KS-structured covariance model which corresponds to an errors-in-variables model. The KS (inverse) covariance structure corresponds to the Cartesian product of graphs~\citep{kalaitzis2013bigraphical,greenewald2019tensor}, which leads to more parsimonious representations of (conditional) dependency than the KP. However, unlike the KP model, KS lacks an interpretable generative representation for the data. Recently, \citet{wang20sylvester} proposed a new class of structured graphical models, called the Sylvester graphical models, for tensor-variate data. The resulting inverse covariance matrix has the KS structure in its square-root factors. This square-root KS structure is hinted in the paper to have a connection with certain physical processes, but no illustration is provided.

A common challenge for structured tensor graphical models is the efficient estimation of the underlying (conditional) dependency structures. KP-structured models are generally estimated via extension of GLasso~\citep{friedman2008sparse} that iteratively minimize the $\ell_1$-penalized negative likelihood function for the matrix-normal data with KP covariance. This procedure was shown to converge to some local optimum of the penalized likelihood function~\citep{yin2012model,tsiligkaridis2013convergence}. Similarly, \citet{kalaitzis2013bigraphical} further extended GLasso to the KS-structured case for $2$-way tensor data. \citet{greenewald2019tensor} extended this to multiway tensors, exploiting the linearity of the space of KS-structured matrices and developing a projected proximal gradient algorithm for KS-structured inverse covariance matrix estimation, which achieves linear convergence (i.e., geometric convergence rate) to the global optimum. In~\citet{wang20sylvester}, the Sylvester-structured graphical model is estimated via a nodewise regression approach inspired by algorithms for estimating a class of vector-variate graphical models~\citep{meinshausen2006high,khare2015convex}. However, no theoretical convergence result for the algorithm was established nor did they study the computational efficiency of the algorithm.

In the modern era of big data, both computational and statistical learning accuracy are required of algorithms. Furthermore, when the objective is to learn representations for physical processes, interpretablility is crucial. In this paper, we bridge this ``Statistical-to-Computational-to-Interpretable gap'' for Sylvester graphical models. We develop a simple yet powerful first-order optimization method, based on the Proximal Alternating Linearized Minimization (PALM) algorithm, for recovering the conditional dependency structure of such models. Moreover, we provide the link between the Sylvester graphical models and physical processes obeying differential equations and illustrate the link with a real-data example. The following are our principal contributions:
\begin{enumerate}
    \item A fast algorithm that efficiently recovers the generating factors of a representation for high-dimensional multiway data, significantly improving on~\citet{wang20sylvester}.
    \item A comprehensive convergence analysis showing linear convergence of the objective function to its global optimum and providing insights for choices of hyperparameters.
    \item A novel application of the algorithm to an important multi-modal solar flare prediction problem from solar magnetic field sequences. For such problems, SG-PALM is physically interpretable in terms of the partial differential equations governing solar activities proposed by heliophysicists.
\end{enumerate}

\section{Background and Notation}\label{sec:background}
\subsection{Notations}
In this paper, scalar, vector and matrix quantities are denoted by lowercase letters, boldface lowercase letters and boldface capital letters, respectively. For a matrix $\mat{A} = (\mat{A}_{i,j}) \in \mathbb{R}^{d \times d}$, we denote $\|\mat{A}\|_2, \|\mat{A}\|_F$ as its spectral and Frobenius norm, respectively. We define $\|\mat{A}\|_{1,\text{off}} := \sum_{i \neq j} |\mat{A}_{i,j}|$ as its off-diagonal $\ell_1$ norm. For tensor algebra, we adopt the notations used by \citet{kolda2009tensor}. A $K$-th order tensor is denoted by boldface Euler script letters, e.g, $\tensor{X} \in \R^{d_1 \times \dots \times d_K}$. The $(i_1,\dots, i_K)$-th element of $\tensor{X}$ is denoted by $\tensor{X}_{i_1,\dots, i_K}$, and the vectorization of $\tensor{X}$ is the $d$-dimensional vector $\vecto(\tensor{X}) := (\tensor{X}_{1,1,\dots,1},\tensor{X}_{2,1,\dots,1},\dots,\tensor{X}_{d_1,1,\dots,1},\dots,\tensor{X}_{d_1,d_2,\dots,d_k})^T$ with $d=\prod_{k=1}^K d_k$. A fiber is the higher order analogue of the row and column of matrices. It is obtained by fixing all but one of the indices of the tensor. Matricization, also known as unfolding, is the process of transforming a tensor into a matrix. The mode-$k$ matricization of a tensor $\tensor{X}$, denoted by $\tensor{X}_{(k)}$, arranges the mode-$k$ fibers to be the columns of the resulting matrix. The $k$-mode product of a tensor $\tensor{X} \in \R^{d_1 \times \dots \times d_K}$ and a matrix $\mat{A} \in \R^{J \times d_k}$, denoted as $\tensor{X} \times_k \mat{A}$, is of size $d_1 \times \dots \times d_{k-1} \times J \times d_{k+1} \times \dots d_K$. Its entry is defined as $(\tensor{X} \times_k \mat{A})_{i_1,\dots,i_{k-1},j,i_{k+1},\dots,i_K} := \sum_{i_k=1}^{d_k} \tensor{X}_{i_1,\dots,i_K} A_{j,i_k}$. For a list of matrices $\{\mat{A}_k\}_{k=1}^K$ with $\mat{A}_k \in \R^{d_k \times d_k}$, we define $\tensor{X} \times \{\mat{A}_1,\dots,\mat{A}_K\} := \tensor{X} \times_1 \mat{A}_1 \times_2 \dots \times_K \mat{A}_K$. Lastly, we define the $K$-way Kronecker product as $\bigotimes_{k=1}^K \mat{A}_k = \mat{A}_1 \otimes \cdots \otimes \mat{A}_K$, and the equivalent notation for the Kronecker sum as $\bigoplus_{k=1}^K \mat{A}_k = \mat{A}_1 \oplus \dots \oplus \mat{A}_K = \sum_{k=1}^K \mat I_{[d_{k+1:K}]} \otimes \mat{A}_k \otimes \mat I_{[d_{1:k-1}]}$, where $\mat I_{[d_{k:\ell}]} = \mat I_{d_k} \otimes \dots \otimes \mat I_{d_\ell}$. For the case of $K=2$, $\mat{A}_1 \oplus \mat{A}_2 = \mat{I}_{d_2} \otimes \mat{A}_1 + \mat{A}_2 \otimes \mat{I}_{d_1}$.

\subsection{Tensor Graphical Models}
A random tensor $\tensor{X} \in \reals^{d_1 \times \dots \times d_K}$ follows the tensor normal distribution with zero mean when $\vecto(\tensor{X})$ follows a normal distribution with mean $\mat{0} \in \reals^d$ and precision matrix $\mat\Omega := \mat\Omega(\mat\Psi_1,\dots,\mat\Psi_K)$, where $d=\prod_{k=1}^K d_k$. Here, $\mat\Omega(\mat\Psi_1,\dots,\mat\Psi_K)$ is parameterized by $\mat\Psi_k \in \reals^{d_k \times d_k}$ via either Kronecker product, Kronecker sum, or the Sylvester structure, and the corresponding negative log-likelihood function (assuming $N$ independent observations $\tensor{X}^i, i=1,\dots,N$)
\begin{equation}\label{eqn:gaussiannegloglik}
     -\frac{N}{2} \log|\mat\Omega| + \frac{N}{2}\tr(\mat{S}\mat\Omega),
\end{equation}
where $\mat\Omega = \bigotimes_{k=1}^K \mat\Psi_k$, $\bigoplus_{k=1}^K \mat\Psi_k$, or $\Big(\bigoplus_{k=1}^K \mat\Psi_k\Big)^2$ for KP, KS, and Sylvester models, respectively; and $\mat{S} = \frac{1}{N}\sum_{i=1}^N \vecto(\tensor{X}^i) \vecto(\tensor{X}^i)^T$. To encourage sparsity, penalized negative log-likelihood function is proposed
\begin{equation*}
     -\frac{N}{2} \log|\mat\Omega| + \frac{N}{2}\tr(\mat{S}\mat\Omega) + \sum_{k=1}^K P_{\lambda_k}(\mat\Psi_k),
\end{equation*}
where $P_{\lambda_k}(\cdot)$ is a penalty function indexed by the tuning parameter $\lambda_k$ and is applied elementwise to the off-diagonal elements of $\mat\Psi_k$. Popular choices for $P_{\lambda_k}(\cdot)$ include the lasso penalty~\citep{tibshirani1996regression}, the adaptive lasso penalty~\citep{zou2006adaptive}, the SCAD penalty~\citep{fan2001variable}, and the MCP penalty~\citep{zhang2010nearly}. 

\subsection{The Sylvester Generating Equation}
\citet{wang20sylvester} proposed a Sylvester graphical model that uses the Sylvester tensor equation to define a generative process for the underlying multivariate tensor data. The Sylvester tensor equation has been studied in the context of finite-difference discretization of high-dimensional elliptical partial differential equations~\citep{grasedyck2004existence,kressner2010krylov}. Any solution $\tensor{X}$ to such a PDE  must have the (discretized) form:
\begin{equation}\label{eqn:sylvester}
    \begin{aligned}
        \sum_{k=1}^K \tensor{X} \times_k \mat\Psi_k = \tensor{T} &\Longleftrightarrow  \Big(\bigoplus_{k=1}^K \mat\Psi_k \Big) \vecto(\tensor{X}) = \vecto(\tensor{T}).
    \end{aligned}
\end{equation} 
where $\tensor{T}$ is the driving source on the domain, and $\bigoplus_{k=1}^K \mat\Psi_k$ is a Kronecker sum of $\mat\Psi_k$'s representing the discretized differential operators for the PDE, e.g., Laplacian, Euler-Lagrange operators, and associated coefficients. These operators are often sparse and structured. 

For example, consider a physical process characterized as a function $u$ that satisfies:
\begin{equation*}
    \mathcal{D}u = f \quad \text{in} \quad \Omega, \quad u(\Gamma)=0, \quad \Gamma = \partial \Omega.
\end{equation*}
where $f$ is a driving process, e.g., a Wiener process (white Gaussian noise); $\mathcal{D}$ is a differential operator, e.g, Laplacian, Euler-Lagrange; $\Omega$ is the domain; and $\Gamma$ is the boundary of $\Omega$. After discretization, this is equivalent to (ignoring discretization error) the matrix equation
\begin{equation*}
    \mat{D}\mat{u} = \mat{f}.
\end{equation*}
Here, $\mat{D}$ is a sparse matrix since $\mathcal{D}$ is an infinitesimal operator. Additionally, $\mat{D}$ admits Kronecker structure as a mixture of Kronecker sums and Kronecker products.

The matrix $\mat{D}$ reduces to a Kronecker sum when $\mathcal{D}$ involves no mixed derivatives. For instance, consider the Poisson equation in 2D, where $u(x,y)$ on $[0,1]^2$ satisfies the elliptical PDE
\begin{equation*}
    \mathcal{D}u = (\partial^2_x + \partial^2_y)u = f.
\end{equation*}
The Poisson equation governs many physical processes, e.g., electromagnetic induction, heat transfer, convection, etc. A simple Euler discretization yields $\mat{U} = (u(i,j))_{i,j}$, where $u(i,j)$ satisfies the local equation (up to a constant discretization scale factor)
\begin{equation*}
\begin{aligned}
    4 u(i,j) &= u(i+1,j) + u(i-1,j) + u(i,j+1)  \\
    & \quad + u(i,j-1) - f(i,j).  
\end{aligned}
\end{equation*}
Defining $\mat{u}=\text{vec}(\mat{U})$ and $\mat{A}$ (a tridiagonal matrix)
\begin{equation*}
\mat{A} = 
    \begin{bmatrix}
    -1 & 2 & -1 & & & \\
      & \ddots & \ddots & \ddots \\
      &  & \ddots & \ddots & \ddots \\
      &  &  & -1 & 2 & -1
    \end{bmatrix},
\end{equation*}
then $(\mat{A} \oplus \mat{A})\mat{u} = \mat{f}$, which is the Sylvester equation ($K=2$).

For the Poisson example, if the source $\mat{f}$ is a white noise random variable, i.e., its covariance matrix is proportional to the identity matrix, then the inverse covariance matrix of $\mat{u}$ has sparse square-root factors, since $\text{Cov}^{-1}(\mat{u})=(\mat{A} \oplus \mat{A})(\mat{A} \oplus \mat{A})^T$. Other physical processes that are generated from differential equations will also have sparse inverse covariance matrices, as a result of the sparsity of general discretized differential operators. Note that similar connections between continuous state physical processes and sparse ``discretized'' statistical models have been established by \citet{lindgren2011explicit}, who elucidated a link between Gaussian fields and Gaussian Markov Random Fields via stochastic partial differential equations. 

The Sylvester generative (SG) model~\eqref{eqn:sylvester} leads to a tensor-valued random variable $\tensor{X}$ with a precision matrix $\mat\Omega=\Big(\bigoplus_{k=1}^K \mat\Psi_k\Big)^2$, given that $\tensor{T}$ is white Gaussian. The Sylvester generating factors $\mat\Psi_k$'s can be obtained via minimization of the penalized negative log-pseudolikelihood
\begin{equation}
  \label{eqn:objective}
  \begin{aligned}
    \mathcal{L}_{\mat\lambda}(\mat\Psi)
    = & -\frac{N}{2} \log | (\bigoplus_{k=1}^K \diag(\mat\Psi_k))^2| \\
    & + \frac{N}{2} \tr(\mat{S} \cdot (\bigoplus_{k=1}^K \mat\Psi_k)^2) + \sum_{k=1}^K \lambda_k \|\mat\Psi_k\|_{1, \text{off}}.
  \end{aligned}
\end{equation}
This differs from the true penalized Gaussian negative log-likelihood in the exclusion of off-diagonals of $\mat\Psi_k$'s in the log-determinant term. \eqref{eqn:objective} is motivated and derived directly using the Sylvester equation defined in~\eqref{eqn:sylvester}, from the perspective of solving a sparse linear system. This maximum pseudolikelihood estimation procedure has been applied to vector-variate Gaussian graphical models (see \citet{khare2015convex} and references therein). Detailed derivations and further discussions are provided in Appendix~\ref{supp:pseudolik}.

\section{The SG-PALM Algorithm}\label{sec:method}
Estimation of the generating parameters $\mat\Psi_k$'s of the SG model is challenging since the sparsity penalty applies to the square root factors of the precision matrix, which leads to a highly coupled likelihood function. \citet{wang20sylvester} proposed an estimation procedure called SyGlasso, that recovers only the off-diagonal elements of each Sylvester factor. This is a deficiency in many applications where the factor-wise variances are desired. Moreover, the convergence rate of the cyclic coordinate-wise algorithm used in SyGlasso is unknown and the computational complexity of the algorithm is higher than other sparse Glasso-type procedures. To overcome these deficiencies, we propose a proximal alternating linearized minimization method that is more flexible and versatile, called SG-PALM, for finding the minimizer of \eqref{eqn:objective}. SG-PALM is designed to exploit structures of the coupled objective function and yields simultaneous estimates for both off-diagonal and diagonal entries.

The PALM algorithm was originally proposed to solve nonconvex optimization problems with separable structures, such as those arising in nonnegative matrix factorization~\citep{xu2013block,bolte2014proximal}. Its efficacy in solving convex problems has also been established, for example, in regularized linear regression problems~\citep{shefi2016rate}, it was proposed as an attractive alternative to iterative soft-thresholding algorithms (ISTA). 
The SG-PALM procedure is summarized in Algorithm~\ref{alg:sg-palm}.

For clarity of notation we write
\begin{equation}\label{eqn:decomp_obj}
    \mathcal{L}_{\mat\lambda}(\mat\Psi_1,\dots,\mat\Psi_K) = H(\mat\Psi_1,\dots,\mat\Psi_K) + \sum_{k=1}^K G_k(\mat\Psi_k),
\end{equation}
where $H: \mathbb{R}^{d_1 \times d_1} \times \cdots \times \mathbb{R}^{d_K \times d_K} \rightarrow \mathbb{R}$ represents the log-determinant plus trace terms in \eqref{eqn:objective} and $G_k: \mathbb{R}^{d_k \times d_k} \rightarrow (-\infty,+\infty]$ represents the penalty term in \eqref{eqn:objective} for each axis $k=1,\dots,K$. For notational simplicity we use $\mat\Psi$ (i.e., omitting the subscript) to denote the set $\{\mat\Psi_k\}_{k=1}^K$ or the $K$-tuple $(\mat\Psi_1,\dots,\mat\Psi_K)$ whenever there is no risk of confusion. The gradient of the smooth function $H$ with respect to $\mat\Psi_k$, $\nabla_k H(\mat\Psi)$, is given by
\begin{equation}\label{eqn:block-grad}
\begin{aligned}
    & \diag\Big(\Big\{\tr[(\diag((\mat\Psi_k)_{ii}) + \bigoplus_{j \neq k}\diag(\mat\Psi_j))^{-1}] \Big\}_{i=1}^{d_k} \Big) \\
    & \quad + \mat{S}_k\mat\Psi_k + \mat\Psi_k\mat{S}_k + 2\sum_{j \neq k}\mat{S}_{j,k}.
\end{aligned}
\end{equation}
Here, the first ``$\diag$'' maps a $d_k$-vector to a $d_k \times d_k$ diagonal matrix, the second one maps a scalar (i.e., $(\mat\Psi_k)_{ii}$) to a $(\prod_{j \neq k}d_j) \times (\prod_{j \neq k}d_j)$ diagonal matrix with the same elements, and the third operator maps a symmetric matrix to a matrix containing only its diagonal elements. In addition, we define:
\begin{equation}
    \begin{aligned}
        & \mat{S}_k = \frac{1}{N}\sum_{i=1}^N \tensor{X}_{(k)}^i(\tensor{X}_{(k)}^i)^T, \\
        & \mat{S}_{j,k} = \frac{1}{N}\sum_{i=1}^N \mat{V}_{j,k}^i(\mat{V}_{j,k}^i)^T, \\
        & \mat{V}_{j,k}^i = \tensor{X}_{(k)}^i\Big(\mat{I}_{d_{1:j-1}} \otimes \mat\Psi_j \otimes \mat{I}_{d_{j:K}}\Big)^T, \quad j \neq k.
    \end{aligned}
\end{equation}
A key ingredient of the PALM algorithm is a proximal operator associated with the non-smooth part of the objective, i.e., $G_k$'s. In general, the proximal operator of a proper, lower semi-continuous convex function $f$ from a Hilbert space $\mathcal{H}$ to the extended reals $(-\infty,+\infty]$ is defined by~\citep{parikh2014proximal}
\begin{equation*}
    \vspace{-1pt}
    \text{prox}_f(v) = \argmin_{x \in \mathcal{H}} f(x) + \frac{1}{2}\|x-v\|^2_2
\end{equation*}
for any $v \in \mathcal{H}$. The proximal operator well-defined as the expression on the right-hand side above has a unique minimizer for any function in this class. For $\ell_1$-regularized cases, the proximal operator for the function $G_k$ is given by
\begin{equation}
     \text{prox}_{G_k}^{\lambda_k}(\mat\Psi_k) = \diag(\mat\Psi_k) + \text{soft}(\mat\Psi_k-\diag(\mat\Psi_k), \lambda_k),
\end{equation}
where the soft-thresholding operator $\text{soft}_{\lambda}(x) = \text{sign}(x)\max(|x|-\lambda,0)$ has been applied element-wise. For popular choices of non-convex $G_k$, the proximal operators are derived in Appendix~\ref{supp:nonconvex}.

\begin{algorithm}[!tbh]
\begin{algorithmic}
\caption{SG-PALM}\label{alg:sg-palm}
\REQUIRE Data tensor $\tensor{X}$, mode-$k$ Gram matrix $\mat{S}_k$, regularizing parameter $\lambda_k$, backtracking constant $c \in (0,1)$, initial step size $\eta_0$, initial iterate $\mat\Psi_k$ for each $k=1,\dots,K$.
\WHILE{not converged}
    \FOR{$k=1,\dots,K$}
        \STATE \textit{Line search:} 
        
        Let $\eta^t_k$ be the largest element of $\{c^j \eta_{k,0}^t\}_{j=1,\dots}$ such that condition~\eqref{eqn:linesearch-cond} is satisfied.
        \STATE \textit{Update:} 
        
        $\mat\Psi_k^{t+1} \leftarrow \text{prox}_{\eta^t_k\lambda_k}^{G_k}\Big(\mat\Psi_k^t - \eta^t_k \nabla_k H(\mat\Psi_{i < k}^{t+1},\mat\Psi_{i \geq k}^t)\Big)$.
    \ENDFOR
    \STATE \textit{Update initial step size:} Compute Barzilai-Borwein step size $\eta_0^{t+1}=\min_k \eta^{t+1}_{k,0}$, where $\eta^{t+1}_{k,0}$ is computed via~\eqref{eqn:bb-step}.
\ENDWHILE
\ENSURE Final iterates $\{\mat\Psi_k\}_{k=1}^K$.
\end{algorithmic}
\end{algorithm}

\subsection{Choice of Step Size}
In the absence of a good estimate of the blockwise Lipchitz constant, the step size of each iteration of SG-PALM is chosen using backtracking line search, which, at iteration $t$, starts with an initial step size $\eta_0^t$ and reduces the size with a constant factor $c \in (0,1)$ until the new iterate satisfies the sufficient descent condition:
\begin{equation}\label{eqn:linesearch-cond}
    H(\mat\Psi_{i \leq k}^{t+1},\mat\Psi_{i > k}^t) \leq Q_{\eta^t}(\mat\Psi_{i \leq k}^{t+1},\mat\Psi_{i > k}^t;\mat\Psi_{i < k}^{t+1},\mat\Psi_{i \geq k}^t).
\end{equation}
Here, 
\begin{equation*}
\begin{aligned}
    & Q_{\eta}(\mat\Psi_{i < k},\mat\Psi_k,\mat\Psi_{i > k};\mat\Psi_{i < k},\mat\Psi_k',\mat\Psi_{i > k}) \\
    &= H(\mat\Psi_{i < k},\mat\Psi_k,\mat\Psi_{i > k}) \\ 
    &+ \tr\Big((\mat\Psi_k'-\mat\Psi_k)^T \nabla_k H(\mat\Psi_{i < k},\mat\Psi_k,\mat\Psi_{i > k})\Big) \\
    &+ \frac{1}{2\eta}\|\mat\Psi_k'-\mat\Psi_k\|_F^2.
\end{aligned}
\end{equation*}
The sufficient descent condition is satisfied with any $\frac{1}{\eta}=M_k$ and $M_k \geq L_k$, for any function that has a block-wise Lipschitz gradient with constant $L_k$ for $k=1,\dots,K$. In other words, so long as the function $H$ has block-wise gradient that is Lipschitz continuous with some block Lipschitz constant $L_k>0$ for each $k$, then at each iteration $t$, we can always find an $\eta^t$ such that the inequality in \eqref{eqn:linesearch-cond} is satisfied. Indeed, we proved in Lemma~\ref{lemma:lip} in the Appendix that $H$ has the desired properties. Additionally, in the proof of Theorem~\ref{thm:sg-palm-main} we also showed that the step size found at each iteration $t$ satisfies $\frac{1}{\eta_k^{0}} \leq L_k \leq \frac{1}{\eta_k^{t}} \leq c L_k$.

In terms of the initialization, a safe step size (i.e., very small $\eta_0^t$) often leads to slower convergence. Thus, we use the more aggressive Barzilai-Borwein (BB) step~\citep{barzilai1988two} to set a starting $\eta_0^t$ at each iteration (see Appendix~\ref{supp:bb-step-size} for justifications of the BB method). In our case, for each $k$, the step size is given by
\begin{equation}\label{eqn:bb-step}
    \eta_{k,0}^t = \frac{\|\mat\Psi_k^{t+1}-\mat\Psi_k^{t}\|_F^2}{\tr(\mat{A})},
\end{equation}
where
\begin{equation*}
\begin{aligned}
    \mat{A} &= (\mat\Psi_k^{t+1}-\mat\Psi_k^{t})^T \times \\
    &(\nabla_k H(\mat\Psi_{i \leq k}^{t+1},\mat\Psi_{i > k}^t)
    - \nabla_k H(\mat\Psi_{i < k}^{t+1},\mat\Psi_{i \geq k}^t)).
\end{aligned}
\end{equation*}

\subsection{Computational Complexity}
After pre-computing $\mat{S}_k$, the most significant computation for each iteration in the SG-PALM algorithm is the sparse matrix-matrix multiplications $\mat{S}_k \mat\Psi_k$ and $\mat{S}_{j,k}$ in the gradient calculation. In terms of computational complexity, if $s_k$ is the number of non-zeros per column in $\mat\Psi_k$, then the former and latter can be computed using $O(s_k d_k^2)$ and $O(N\sum_{j \neq k} s_jd_j^2)$ operations, respectively. Thus, each iteration of SG-PALM can be computed using $O\Big(\sum_{k=1}^K (s_k d_k^2 + N\sum_{j \neq k} s_jd_j^2) \Big)$ floating point operations, which is significantly lower than competing methods.

For instance, other popular algorithms for tensor-variate graphical models, such as the TG-ISTA presented in \citet{greenewald2019tensor} and the Tlasso proposed in \citet{lyu2019tensor} both require inversion of $d_k \times d_k$ matrices, which is non-parallelizable and requires $O(d_k^3)$ operations for each $k$. In particular, TeraLasso's TG-ISTA algorithm requires $O\Big(Kd + \sum_{k=1}^K d_k^3\Big)$ operations. The TG-ISTA algorithm requires matrix inversions that cannot easily exploit the sparsity of $\mat\Psi_k$'s. In the sample-starved ultra-sparse setting ($N \ll d$ and $s_k \ll d_k$), the $O(N\sum_{j \neq k} s_jd_j^2)$ terms in SG-PALM are comparable to $O(Kd)$ in TG-ISTA, making it more appealing. The cyclic coordinate-wise method proposed in~\citep{wang20sylvester} does not allow for parallelization since it requires cycling through entries of each $\mat\Psi_k$ in specified order. In contrast, SG-PALM can be implemented in parallel to distribute the sparse matrix-matrix multiplications because at no step do the algorithms require storing all dense matrices on a single machine. 

\section{Convergence Analysis}\label{sec:theory}
In this section, we present the main convergence theorems. Detailed proofs are included in the supplement. Here, we study the statistical convergence behavior for the Sylvester graphical model with an $\ell_1$ penalty function. The convergence behavior of the SG-PALM iterates is presented for convex cases but similar convergence rate can be established for non-convex penalties (see Appendix~\ref{supp:nonconvex}).

We first establish statistical convergence of a global minimizer $\hat{\mat\Psi}$ of \eqref{eqn:objective} to its true value, denoted as $\bar{\mat\Psi}$, under the correct statistical model.
\begin{theorem}\label{thm:statistical}
Let $\mathcal{A}_{k}:=\{(i,j):(\bar{\mat\Psi}_k)_{i,j} \neq 0, i \neq j\}$ and $q_{k}:=|\mathcal{A}_{k}|$ for $k=1,\dots,K$. If $N > O(\max_k q_k d_k \log d)$ and $d:=d_N=O(N^{\kappa})$ for some $\kappa \geq 0$, and further, if the penalty parameter satisfies $\lambda_k:=\lambda_{N,k}=O(\sqrt{\frac{d_k\log d}{N}})$ for all $k=1,\dots,K$, then under conditions (A1-A3) in Appendix~\ref{supp:thm_statistical}, there exists a constant $C>0$ such that for any $\eta>0$ the following events hold with probability at least $1 - O(\exp(-\eta \log d))$:
  \begin{equation*}
    \begin{aligned}
        & \sum_{k=1}^K\|\text{offdiag}(\hat{\mat\Psi}_k) - \text{offdiag}(\bar{\mat\Psi}_k)\|_F \\
        & \leq C\sqrt{K}\max_{k}\sqrt{q_{k}}\lambda_{k}.  
    \end{aligned}
  \end{equation*}
Here $\text{offdiag}(\mat\Psi_k)$ contains only the off-diagonal elements of $\mat\Psi_k$. If further $\min_{(i,j) \in \mathcal{A}_{k}}|(\bar{\mat{\Psi}}_k)_{i,j}| \geq 2C\max_{k}\sqrt{q_{k}}\lambda_{k}$ for each $k$, then sign($\hat{\mat{\Psi}}_k$)=sign($\bar{\mat{\Psi}}_k$).
\end{theorem}

Theorem~\ref{thm:statistical} means that under regularity conditions on the true generative model, and with appropriately chosen penalty parameters $\lambda_k$'s guided by the theorem, one is guaranteed to recover the true structures of the underlying Sylvester generating parameters $\mat\Psi_k$ for $k=1,\dots,K$ with probability one, as the sample size and dimension grow.

We next turn to convergence of the iterates $\{\mat\Psi^t\}$ from SG-PALM to a global optimum of \eqref{eqn:objective}. 

\begin{theorem}\label{thm:sg-palm-main}
The 
Let $\{\mat\Psi^{(t)}\}_{t \geq 0}$ be generated by SG-PALM.
Then, SG-PALM converges in the sense that
\begin{equation*}
\begin{aligned}
    & \frac{\mathcal{L}_{\mat\lambda}(\mat\Psi^{(t+1)}) - \min \mathcal{L}_{\mat\lambda}}{\mathcal{L}_{\mat\lambda}(\mat\Psi^{(t)}) - \min \mathcal{L}_{\mat\lambda}} \\
    & \leq \Bigg(\frac{\alpha^2L_{\min}}{4Kc^2(\sum_{j=1}^K L_j)^2 + 4c^2L_{\max}} + 1\Bigg)^{-1},
\end{aligned}
\end{equation*}
where $\alpha$, $L_k,k=1,\dots,K$ are positive constants, $L_{\min}=\min_jL_j$, $L_{\max}=\max_jL_j$, and $c \in (0,1)$ is the backtracking constant defined in Algorithm~\ref{alg:sg-palm}.
\end{theorem}
Note that the term on  the right hand side of the inequality above is strictly less than $1$. This means that the SG-PALM algorithm converges linearly, which is a strong results for a non-strongly convex objective (i.e., $\mathcal{L}_{\mat\lambda}$). Although similar convergence behaviors of the PALM-type algorithms have been studied for other problems~\citep{xu2013block,bolte2014proximal}, such as nonnegative matrix/tensor factorization, the analysis of this paper works for non-strongly block multi-convex objectives, leveraging more recent analyses of multi-block PALM and a class of functions satisfying the the Kurdyka - \L ojasiewicz (KL) property (defined in Section~\ref{supp:proofs} of the Appendix). To the best of our knowledge, for first-order optimization methods, our rate is faster than any other Gaussian graphical models having non-strongly convex objectives (see \citet{khare2015convex,oh2014optimization} and references therein) and comparable with those having strongly-convex objectives (see, for example, \citet{guillot2012iterative,dalal2017sparse,greenewald2019tensor}).

\section{Experiments}\label{sec:experiments}
Experiments in this section were performed in a system with \texttt{8-core Intel Xeon CPU E5-2687W v2 3.40GHz} equipped with \texttt{64GB RAM}. Both SG-PALM and SyGlasso were implemented in \texttt{Julia v1.5} (\url{https://github.com/ywa136/sg-palm}). For real data analyses, we used the \texttt{Tlasso} package implementation in \texttt{R}~\citep{r-Tlasso} and the TeraLasso implementation in \texttt{MATLAB} (\url{https://github.com/kgreenewald/teralasso}). 

\subsection{Synthetic Data}
We first validate the convergence theorems discussed in the previous section via simulation studies. Synthetic datasets were generated from true sparse Sylvester factors $\{\mat\Psi_k\}_{k=1}^K$ where $K=\{2,3\}$ and $d_k=\{16,32,64,128\}$ for all $k$. Instances of the random matrices used here have uniformly random sparsity patterns with edge densities (i.e., the proportion of non-zero entries) ranging from $0.1\% -30\%$ on average over all $\mat\Psi_k$'s. For each $d$ and edge density combination, random samples of size $N=\{10,100,1000\}$ were tested. For comparison, the initial iterates, convergence criteria were matched between SyGlasso and SG-PALM. Highlights of the results in run times are summarized in Table~\ref{tab:synthetic_run_time}.

\begin{table}[tbh!]
\centering
\caption{Run time comparisons (in seconds with N/As indicating those exceeding $24$ hour) between SyGlasso and SG-PALM on synthetic datasets with different dimensions, sample sizes, and densities of the generating Sylvester factors. Note that the proposed SG-PALM has average speed-up ratios ranging from $1.5$ to $10$ over SyGlasso.}
\label{tab:synthetic_run_time}
\begin{tabular}{|c|p{0.5cm}|p{0.6cm}||r|r|}
 \multicolumn{5}{c}{} \\
 \hline
 \multirow{2}{*}{$d$} & \multirow{2}{*}{$N$} & \multirow{2}{*}{NZ\%} & \textbf{SyGlasso} & \textbf{SG-PALM} \\
 \cline{4-5} 
 &&& \textbf{iter} \quad \textbf{sec} & \textbf{iter} \quad \textbf{sec} \\
 \hline
 \multirow{6}{*}{$128^2$} & \multirow{2}{*}{$10^1$} & $1.20$ &
 $17$ \quad $138.5$ & $46$ \quad $5.8$ \\
 && $24.0$ &
 $20$ \quad $169.3$ & $48$ \quad $6.2$ \\
 \cline{2-5}
 & \multirow{2}{*}{$10^2$} & $1.30$ & 
 $21$ \quad $211.3$ & $50$ \quad $12.6$ \\
 && $27.0$ &
 $30$ \quad $303.6$ & $47$ \quad $21.9$ \\
 \cline{2-5}
 & \multirow{2}{*}{$10^3$} & $1.30$ & 
 $21$ \quad $2045.8$ & $50$ \quad $80.1$ \\
 && $25.0$ &
 $47$ \quad $4782.7$ & $51$ \quad $373.1$ \\
 \cline{1-5}
 \multirow{6}{*}{$16^3$} & \multirow{2}{*}{$10^1$} & $0.11$ &
 $9$ \quad $4.6$ & $11$ \quad $4.5$ \\
 && $4.10$ &
 $9$ \quad $5.1$ & $32$ \quad $5.1$ \\
 \cline{2-5}
 & \multirow{2}{*}{$10^2$} & $0.21$ & 
 $8$ \quad $8.8$ & $11$ \quad $5.4$ \\
 && $2.60$ &
 $8$ \quad $10.8$ & $35$ \quad $7.2$ \\
 \cline{2-5}
 & \multirow{2}{*}{$10^3$} & $0.26$ & 
 $8$ \quad $82.4$ & $12$ \quad $14.3$ \\
 && $3.40$ &
 $10$ \quad $99.2$ & $37$ \quad $33.5$ \\
 \cline{1-5}
 \multirow{6}{*}{$32^3$} & \multirow{2}{*}{$10^1$} & $0.13$ &
 $10$ \quad $191.2$ & $19$ \quad $7.3$ \\
 && $7.50$ &
 $17$ \quad $304.8$ & $42$ \quad $10.2$ \\
 \cline{2-5}
 & \multirow{2}{*}{$10^2$} & $0.46$ & 
 $9$ \quad $222.4$ & $24$ \quad $28.9$ \\
 && $7.00$ &
 $17$ \quad $395.2$ & $41$ \quad $48.5$ \\
 \cline{2-5}
 & \multirow{2}{*}{$10^3$} & $0.10$ & 
 $9$ \quad $1764.8$ & $22$ \quad $226.4$ \\
 && $6.90$ &
 $19$ \quad $3789.4$ & $41$ \quad $473.9$ \\
 \cline{1-5}
 \multirow{6}{*}{$64^3$} & \multirow{2}{*}{$10^1$} & $0.65$ &
 $10$ \quad $583.7$ & $42$ \quad $91.3$ \\
 && $14.5$ &
 $22$ \quad $952.2$ & $47$ \quad $119.0$ \\
 \cline{2-5}
 & \multirow{2}{*}{$10^2$} & $0.62$ & 
 $9$ \quad $6683.7$ & $41$ \quad $713.9$ \\
 && $14.4$ &
 $21$ \quad $15607.2$ & $48$ \quad $1450.9$ \\
 \cline{2-5}
 & \multirow{2}{*}{$10^3$} & $0.85$ & 
 N/A  & $39$ \quad $6984.4$ \\
 && $14.0$ &
 N/A  & $48$ \quad $12968.7$ \\
 \hline
\end{tabular}
\end{table}

Convergence behavior of SG-PALM is shown in Figure~\ref{fig:convergence_sg-palm} (a) for the datasets with $d_k=32$, $N=\{10,100\}$, and edge densities roughly around $5\%$ and $20\%$, respectively. Geometric convergence rate of the function value gaps under Theorem~\ref{thm:sg-palm-main} can be verified from the plot. Note an acceleration in the convergence rate (i.e., a steeper slope) near the optimum, which is suggested by the ``localness'' of the KL property of the objective function close to its global optimum. Further for the same datasets, in Figure~\ref{fig:convergence_sg-palm} (b), SG-PALM graph recovery performances is illustrated, where the Matthew's Correlation Coefficients (MCC) is plotted against run time. Here, MCC is defined by 
\begin{equation*}
    \text{MCC} = \frac{\text{TP}\times\text{TN}-\text{FP}\times\text{FN}}{\sqrt{(\text{TP}+\text{FP})(\text{TP}+\text{FN})(\text{TN}+\text{FP})(\text{TN}+\text{FN})}},
\end{equation*}
where TP is the number of true positives, TN the number of true negatives, FP the number of false positives, and FN the number of false negatives of the estimated edges (i.e., non-zero elements of $\mat\Psi_k$'s). An MCC of $1$ represents a perfect prediction, $0$ no better than random prediction and $-1$ indicates total disagreement between prediction and observation. The results validate the statistical accuracy under Theorem~\ref{thm:statistical}. It also shows that SG-PALM outperforms SyGlasso (indicated by blue/red solid dots) within the same time budget.

\begin{figure*}[tbh!]
    \centering
    \begin{subfigure}{0.45\linewidth}
    \centering
    \includegraphics[width=0.85\textwidth]{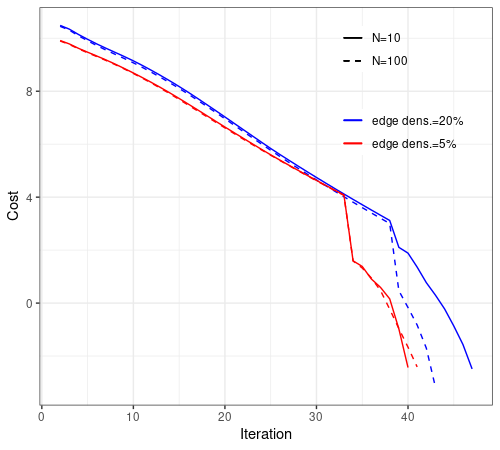}
    \caption{Cost gap vs. Iteration}
    \end{subfigure}
    \hspace{0.1pt}
    \begin{subfigure}{0.45\linewidth}
    \centering
    \includegraphics[width=0.85\textwidth]{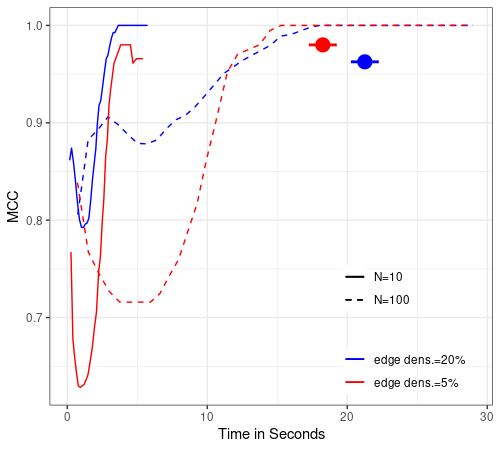}
    \caption{MCC vs. Run time}
    \end{subfigure}
    \caption{Convergence of SG-PALM algorithm under datasets with varying sample sizes (solid and dashed) generated via matrices with different sparsity (red and blue). The function value gaps on log-scale (left) verifies the geometric convergence rate in all cases and the MCC over time (right) demonstrates the algorithm's accuracy and efficiency. Note that the SG-PALM reached almost perfect recoveries (i.e., MCC of $1$) within $20$ seconds in all cases. In comparison, SyGlasso (big solid dots with line-range) was only able to achieve at lower MCCs for lower sample-size cases within $30$ seconds.}
    \label{fig:convergence_sg-palm}
\end{figure*}

\subsection{Solar Flare Imaging Data}
A solar flare occurs when magnetic energy that has built up in the solar atmosphere is suddenly released. Such events strongly influence space weather near the Earth. Therefore, reliable predictions of these flaring events are of great interest.
Recent work~\citep{chen2019identifying,jiao2019solar,sun2019interpreting} has shown the promise of machine learning methods for early forecasting of these events using imaging data from the solar atmosphere. 
In this work, we illustrate the viability of the SG-PALM algorithm for solar flare prediction using data acquired by multiple instruments: the Solar Dynamics Observatory (SDO)/Helioseismic and Magnetic Imager (HMI) and SDO/Atmospheric Imaging Assembly (AIA). It is evident that these data contain information about the physical processes that govern solar activities (see Appendix~\ref{supp:additional_experiments} for detailed data descriptions). 

The data samples are summarized in $d_1 \times d_2 \times d_3 \times d_4$ tensors with $q=d_1 \cdot d_2 \cdot d_3=50 \cdot 100 \cdot 7 = 35000$ and $p=d_4=13$. The first two modes represent the images' heights and widths, the third mode represents the HMI/AIA components/channels, and the last mode represents the length of the temporal window. Previous studies~\citep{chen2019identifying,jiao2019solar} found that the time series of solar images from the SDO/HMI data provide useful information for distinguishing strong solar flares of M/X class from weak flares of A/B class roughly 24 to 12 hours prior to the flare event. Thus, in this study we use a $13$-hour temporal window recorded with $1$-hour cadence, prior to the occurrence of a solar flare. The task is to predict the $p$th frame using the frames in each of the $p-1$ previous hours (i.e., one hour ahead prediction). Each observation is a video with full dimension $d=pq$, and each $p$-dimensional observation vector is formed by concatenating the $p$ time-consecutive $q$-dimensional vectors (vectorization of the matrices representing pixels of the multichannel images) without overlapping the time segments. 
The training set contains two types (B- and MX-class flares) of active regions producing flares. Each is distinguished by the flaring intensities, and there are a total of $186$ B flares and $48$ MX flares. Forward linear predictors  were constructed using estimated precision matrices in a multi-output least squares regression setting. Specifically, we constructed the linear predictor of a frame from the $p-1$ previous frames in the same video:
\begin{equation}
    \hat{\mat{y}}_t = -\mat\Omega_{2,2}^{-1}\mat\Omega_{2,1}\mat{y}_{t-1:t-(p-1)},
\end{equation}
where $\mat{y}_{t-1:t-(p-1)} \in \mathbb{R}^{(p-1)q}$ is the stacked set of pixel values from the previous $p-1$ time instances and $\mat\Omega_{2,1} \in \mathbb{R}^{q \times (p-1)q}$ and $\mat\Omega_{2,2} \in \mathbb{R}^{q \times q}$ are submatrices of the $pq \times pq$ estimated precision matrix:
\begin{equation*}
    \hat{\mat\Omega} =
    \begin{pmatrix}
    \mat\Omega_{1,1} & \mat\Omega_{1,2} \\
    \mat\Omega_{2,1} & \mat\Omega_{2,2}
    \end{pmatrix}.
\end{equation*}
The predictors were tested on the data containing flares observed from different active regions than those in training set, so that the predictor has never ``seen'' the frames that it attempts to predict, corresponding to $117$ observations of which $93$ are B-class flares and $24$ are MX-class flares. Figure~\ref{fig:nrmse_comparison} shows the root mean squared error normalized by the difference between maximum and minimum pixels (NRMSE) over the testing samples, for the forecasts based on the SG-PALM estimator, TeraLasso estimator~\citep{greenewald2019tensor}, Tlasso estimator~\citep{lyu2019tensor}, and IndLasso estimator. Here, the TeraLasso and the Tlasso are estimation algorithms for a KS and a KP tensor precision matrix model, respectively; the IndLasso denotes an estimator obtained by applying independent and separate $\ell_1$-penalized regressions to each pixel in $\mat{y}_t$. The SG-PALM estimator was implemented using a regularization parameter $\lambda_{N}=C_1\sqrt{\frac{\min(d_k)\log(d)}{N}}$ for all $k$ with the constant $C_1$ chosen by optimizing the prediction NRMSE on the training set over a range of $\lambda$ values parameterized by $C_1$. The TeraLasso estimator and the Tlasso estimator were implemented using $\lambda_{N,k}=C_2\sqrt{\frac{\log(d)}{N\prod_{i \neq k}d_i}}$ and $\lambda_{N,k}=C_3\sqrt{\frac{\log(d_k)}{Nd}}$ for $k=1,2,3$, respectively, with $C_2, C_3$ optimized in a similar manner. Each sparse regression in the IndLasso estimator was implemented and tuned independently with regularization parameters chosen from a grid via cross-validation.

We observe that SG-PALM outperforms all three other methods, indicated by NRMSEs across pixels. Figure~\ref{fig:predicted_vs_real_img} depicts examples of predicted images, comparing with the ground truth. The SG-PALM estimates produced most realistic image predictions that capture the spatially varying structures and closely approximate the pixel values (i.e., maintaining contrast ratios). The latter is important as the flares are being classified into weak (B-class) and strong (MX-class) categories based on the brightness of the images, and stronger flares are more likely to lead to catastrophic events, such as those damaging spacecrafts. Lastly, we compare run times of the SG-PALM algorithm for estimating the precision matrix from the solar flare data with SyGlasso. Table~\ref{tab:solar_flare_run_time} in Appendix~\ref{supp:additional_experiments} illustrates that the SG-PALM algorithm converges faster in wallclock time. Note that in this real dataset, which is potentially non-Gaussian, the convergence behavior of the algorithms is different compare to synthetic examples. Nonetheless, SG-PALM enjoys an order of magnitude speed-up over SyGlasso.

\begin{figure*}[!tbh]
\centering
\begin{tabular}{@{}c@{}}
    \quad Avg. NRMSE = $0.0379$, $0.0386$, $0.0579$, $0.1628$ (from left to right) \\
    \rotatebox{90}{\qquad \quad AR B} 
    \includegraphics[width=0.85\textwidth]{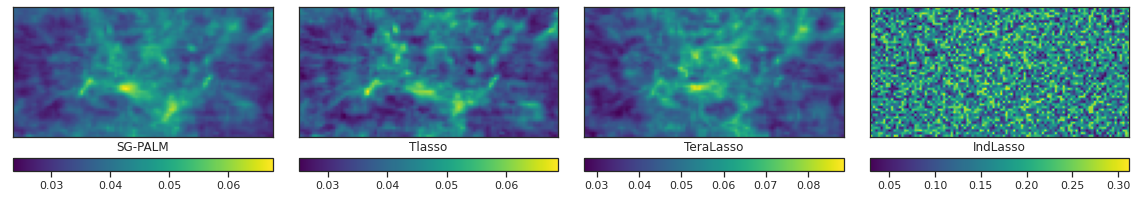}  \\
    \quad Avg. NRMSE = $0.0620$, $0.0790$, $0.0913$, $0.1172$ (from left to right) \\
    \rotatebox{90}{\qquad \quad AR M/X} 
    \includegraphics[width=0.85\textwidth]{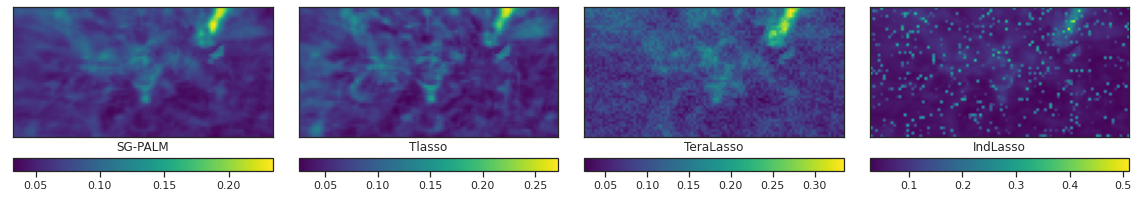} 
\end{tabular}
\caption{Comparison of the SG-PALM, Tlasso, TeraLasso, IndLasso performances measured by NRMSE in predicting the last frame of $13$-frame video sequences leading to B- and MX-class solar flares. The NRMSEs are computed by averaging across testing samples and AIA channels for each pixel. 2D images of NRMSEs are shown to indicate that certain areas on the images (usually associated with the most abrupt changes of the magnetic field/solar atmosphere) are harder to predict than the rest. SG-PALM achieves the best overall NRMSEs across pixels. B flares are generally easier to predict due to both a larger number of samples in the training set and smoother transitions from frame to frame within a video (see the supplemental material for details).}
\label{fig:nrmse_comparison}
\end{figure*}

\begin{figure*}[!tbh]
\centering
\begin{tabular}{@{}c@{}}
    Predicted examples - B vs. M/X \\
    \rotatebox{90}{\qquad AR B}
    \includegraphics[width=0.85\textwidth]{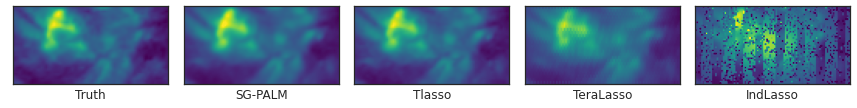}  \\
    \rotatebox{90}{\qquad AR B}
    \includegraphics[width=0.85\textwidth]{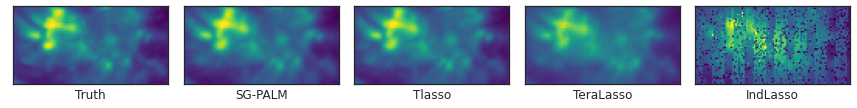} \\
    \rotatebox{90}{\quad AR M/X}
    \includegraphics[width=0.85\textwidth]{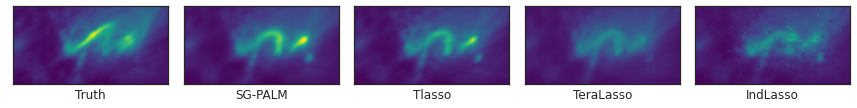} \\
    \rotatebox{90}{\quad AR M/X}
    \includegraphics[width=0.85\textwidth]{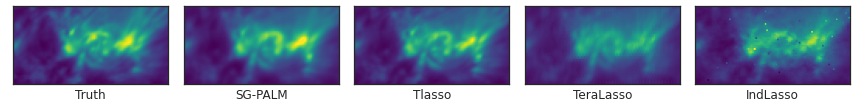}
\end{tabular}
\caption{Examples of one-hour ahead prediction of the first two AIA channels of last frames of $13$-frame videos, leading to B- (first two rows) and MX-class (last two rows) flares, produced by the SG-PALM, Tlasso, TeraLasso, IndLasso algorithms, comparing to the real image (far left column). Note that in general linear forward predictors tend to underestimate the contrast ratio of the images. The proposed SG-PALM produced the best-quality images in terms of both the spatial structures and contrast ratios.See the supplemental material for examples of predicted images from the HMI instrument.}
\label{fig:predicted_vs_real_img}
\end{figure*}

\subsection{Physical Interpretability}
To explain the advantages of the proposed model over other similar models (e.g., Tlasso, TeraLasso), we provide further discussions here on the connection between the Sylvester generating model and PDEs. Consider the 2D spatio-temporal process $u(\mat{x},t)$:  
\begin{equation}\label{eqn:convec-diff}
    \partial u / \partial t = \theta \sum_{i=1}^2 \partial^2 u / \partial x_i^2 + \epsilon \sum_{i=1}^2 \partial u / \partial x_i,
\end{equation}
where $\theta,\epsilon$ are positive real (unknown) coefficients. This is the basic form of a class of parabolic and hyperbolic PDEs, the Convection-Diffusion equation that generalizes the Poisson equation presented in Section~\ref{sec:background} by incorporating temporal evolution. These equations are closely related to the Navier-Stokes equation commonly used in stochastic modelling for weather and climate prediction. Coupled with Maxwell's equations, they can be used to model and study magneto-hydrodynamics~\citep{roberts2006slow}, which characterize solar activities including flares. 

After finite-difference discretization, Equation~\eqref{eqn:convec-diff} is equivalent to the Sylvester matrix equation $\mat{A}_{\theta,\epsilon}\mat{U}_t + \mat{U}_t\mat{A}_{\theta,\epsilon} = \mat{U}_{t-1}$, where $\mat{U}_t=(u((i,j),t))_{ij}$ and $\mat{A}_{\theta,\epsilon}$ is a tridiagonal matrix with values that depend on the coefficients $\theta,\epsilon$ and discretization step sizes. Assuming a linear Gaussian 
state-space model for some observed process $\mat{X}_t$ governed by the Convection-Diffusion dynamics:
\begin{equation*}
    \begin{aligned}
        & \mat{A}_{\theta,\epsilon}\mat{U}_t + \mat{U}_t\mat{A}_{\theta,\epsilon} = \mat{U}_{t-1}, \\
        & \mat{X}_t = \mat{U}_t + \mat{V}_t,
    \end{aligned}
\end{equation*}
where $\mat{V}_t \sim \mathcal{N}(\mat{0},\sigma^2\mat{I})$ is some time-invariant white noise. Then the precision matrix of the true process $\mat{U}_t$ evolves as $\mat\Omega_t = \mat{A}_{\theta,\epsilon} \mat\Omega_{t-1} \mat{A}_{\theta,\epsilon}^T + \sigma^2\mat{I}$. Note that this is not necessarily sparse as assumed by the Sylvester graphical model, but the steady-state precision matrix satisfies $\mat\Omega_{\infty} = \mat{A}_{\theta,\epsilon} \mat\Omega_{\infty} \mat{A}_{\theta,\epsilon}^T + \sigma^2\mat{I}$, which is indeed sparse because $\mat{A}_{\theta,\epsilon}$ is tridiagonal. This strong connection between the Sylvester graphical model and the underlying physical processes governing solar activities make the proposed approach particularly suitable for the case study presented in the previous section. 

Additionally, the learned generating factors $\mat{A}_{\theta,\epsilon}$ could be further used to interpret physical processes that involve both \textit{unknown structure and unknown parameters}. Particularly, in Equation~\eqref{eqn:convec-diff}, the coefficients $\theta$ (diffusion constant) and $\epsilon$ (convective constant) affect the dynamics. Similarly, with the estimated Sylvester generating factors ($\mat\Psi_k$'s), we are not only able to extract the sparsity patterns of the discretized differential operators but also estimate the coefficients of the underlying magneto-hydrodynamics equation for solar flares. Therefore, the SG-PALM can be used as a data-driven method for PDE parameter estimation from physical observations.


\section{Conclusion}\label{sec:conclusion}
We proposed SG-PALM, a proximal alternating linearized minimization method for solving a pseudo-likelihood based sparse tensor-variate Gaussian precision matrix estimation problem. Geometric rate of convergence of the proposed algorithm is established building upon recent advances in the theory of PALM-type algorithms. We demonstrated that SG-PALM outperforms the coordinate-wise minimization method in general, and in ultra-high dimensional settings SG-PALM can be faster by at least an order of magnitude. A link between the Sylvester generating equation underlying the graphical model and the Convection-Diffusion type of PDEs governing certain physical processes was established. This connection was illustrated on a novel astrophysics application, where multi-instrument imaging datasets characterizing solar flare events were used. The proposed methodology was able to robustly forward predict both the patterns and intensities of the solar atmosphere, yielding potential insights to the underlying physical processes that govern the flaring events.

\subsubsection*{Acknowledgements}
The authors thank Zeyu Sun and Xiantong Wang for their help in pre-processing the solar flare datasets. The research was partially supported by US Army grant W911NF-15-1-0479 and NASA grant 80NSSC20K0600.



\clearpage
\bibliographystyle{icml2021}
\bibliography{sg-palm-icml2021-ref}

\clearpage

\appendix
\onecolumn

\section*{Appendix}
\begin{itemize}
    \item[] Section~\ref{supp:pseudolik} provides detailed derivation of the log-pseudolikelihood function.
    \item[] Section~\ref{supp:bb-step-size} provides justifications for the Barzilai-Borwein step sizes implemented in Algorithm~\ref{alg:sg-palm}.
    \item[] Section~\ref{supp:proofs} provides detailed proofs of Theorems~\ref{thm:statistical} and \ref{thm:sg-palm-main}.    
    \item[] Section~\ref{supp:nonconvex} discusses extensions of Algorithm~\ref{alg:sg-palm} and its convergence properties to non-convex cases.
    \item[] Section~\ref{supp:additional_experiments} provides additional details of the solar flare experiments.
\end{itemize}

\section{Derivation of the Log-Pseudolikelihood}\label{supp:pseudolik}
By rewriting the Sylvester tensor equation defined in \eqref{eqn:sylvester} element-wise, we first observe that
\begin{equation}
\label{eqn:elementwise_sylvester}
\begin{aligned}
    & \left( \sum_{k=1}^K (\mat{\Psi}_k)_{i_k,i_k} \right) \tensor{X}_{i_{[1:K]}} \\
    & = -\sum_{k=1}^K \sum_{j_k \neq i_k} (\mat{\Psi}_k)_{i_k,j_k} \tensor{X}_{i_{[1:k]},j_k,i_{[k+1:K]}} + \tensor{T}_{i_{[1:K]}}.
\end{aligned}
\end{equation} 
Note that the left-hand side of \eqref{eqn:elementwise_sylvester} involves only the summation of the diagonals of the $\mat{\Psi}_k$'s and the right-hand side is composed of columns of $\mat\Psi_k$'s that exclude the diagonal terms. Equation \eqref{eqn:elementwise_sylvester} can be interpreted as an autogregressive model relating the $(i_1,\dots,i_K)$-th element of the data tensor (scaled by the sum of diagonals) to other elements in the fibers of the data tensor. The columns of $\mat{\Psi}_k$'s act as regression coefficients. The formulation in \eqref{eqn:elementwise_sylvester} naturally leads to a pseudolikelihood-based estimation procedure \citep{besag1977efficiency} for estimating $\mat\Omega$ (see also \citet{khare2015convex} for how this procedure applied to vector-variate Gaussian graphical model estimation). It is known that inference using pseudo-likelihood is consistent and enjoys the same $\sqrt{N}$ convergence rate as the MLE in general \citep{varin2011overview}. This procedure can also be more robust to model misspecification (e.g., non-Gaussianity) in the sense that it assumes \textit{only that the sub-models/conditional distributions (i.e., $\tensor{X}_i|\tensor{X}_{-i}$) are Gaussian}. Therefore, in practice, even if the data is not Gaussian, the Maximum Pseudolikelihood Estimation procedure is able to perform reasonably well. \citet{wang20sylvester} also studied a different model misspecification scenario where the Kronecker product/sum and Sylvester structures are mismatched for SyGlasso.

From \eqref{eqn:elementwise_sylvester} we can define the sparse least-squares estimators for $\mat\Psi_k$'s as the solution of the following convex optimization problem:
\begin{equation*}
  \begin{aligned}
    & \min_{\substack{\mat{\Psi}_k \in \R^{d_k \times d_k}\\k=1,\dots K}} -N \sum_{i_1,\dots,i_K} \log \tensor{W}_{i_{[1:K]}} \\ 
    & \qquad + \frac{1}{2} \sum_{i_1,\dots,i_K} \|(I) + (II)\|_2^2 + \sum_{k=1}^K P_{\lambda_k}(\mat{\Psi}_k).
  \end{aligned} \
\end{equation*}
where $P_{\lambda_k}(\cdot)$ is a penalty function indexed by the tuning parameter $\lambda_k$ and 
\begin{align*}
  (I) & = \tensor{W}_{i_{[1:K]}}\tensor{X}_{i_{[1:K]}} \\
  (II) & = \sum_{k=1}^K \sum_{j_k \neq i_k} (\mat{\Psi}_k)_{i_k,j_k} \tensor{X}_{i_{[1:k]},j_k,i_{[k+1:K]}},
\end{align*}
with $\tensor{W}_{i_{[1:K]}} := \sum_{k=1}^K (\mat{\Psi}_k)_{i_k,i_k}$.

The optimization problem above can be put into the following matrix form:
\begin{equation*}
    \begin{aligned}
    \min_{\substack{\mat{\Psi}_k \in \R^{d_k \times d_k}\\ k=1,\dots K}} 
    & -\frac{N}{2} \log|(\text{diag}(\mat{\Psi}_1) \oplus \dots \oplus \text{diag}(\mat{\Psi}_K))^2| \\ \nonumber
    + & \frac{N}{2} \tr(\mat{S}(\mat{\Psi}_1 \oplus \dots \oplus \mat{\Psi}_K)^2) + \sum_{k=1}^K P_{\lambda_k}(\mat{\Psi}_k) \nonumber
    \end{aligned}
\end{equation*}
where $\mat{S} \in \R^{d \times d}$ is the sample covariance matrix, i.e., $\mat{S}=\frac{1}{N} \sum_{i=1}^N \vecto(\tensor{X}^i) \vecto(\tensor{X}^i)^T$. Note that this is equivalent to the negative log-pseudolikelihood function that approximates the $\ell_1$-penalized Gaussian negative log-likelihood in the log-determinant term by including only the Kronecker sum of the diagonal matrices instead of the Kronecker sum of the full matrices.

\section{The Barzilai-Borwein Step Size}\label{supp:bb-step-size}
The BB method has been proven to be very successful in solving nonlinear optimization problems. In this section we outline the key ideas behind the BB method, which is motivated by quasi-Newton methods. Suppose we want to solve the unconstrained minimization problem
\begin{equation*}
    \min_x f(x),
\end{equation*}
where $f$ is differentiable. A typical iteration of quasi-Newton methods for solving this problem is
\begin{equation*}
    x_{t+1} = x_{t} - B_t^{-1} \nabla f(x_t),
\end{equation*}
where $B_t$ is an approximation of the Hessian matrix of $f$ at the current iterate $x_t$. Here, $B_t$ must satisfy the so-called secant equation: $B_t s_t = y_t$, where $s_t = x_t - x_{t-1}$ and $y_t = \nabla f(x_t) - \nabla f(x_{t-1})$ for $t \geq 1$. It is noted that in to get $B_t^{-1}$ one needs to solve a linear system, which may be computationally expensive when $B_t$ is large and dense.

One way to alleviate this burden is to use the BB method, which replaces $B_t$ by a scalar matrix $(1/\eta_t)\mat{I}$. However, it is hard to choose a scalar $\eta_t$ such that the secant equation holds with $B_t = (1/\eta_t)\mat{I}$. Instead, one can find $\eta_t$ such that the residual of the secant equation, i.e., $\|(1/\eta_t)s_t - y_t\|_2^2$, is minimized, which leads to the following choice of $\eta_t$:
\begin{equation*}
    \eta_t = \frac{\|s_t\|_2^2}{s_t^Ty_t}.
\end{equation*}
Therefore, a typical iteration of the BB method for solving the original problem is
\begin{equation*}
    x_{t+1} = x_{t} - \eta_t \nabla f(x_t),
\end{equation*}
where $\eta_t$ is computed via the previous formula.

For convergence analysis, generalizations and variants of the BB method, we refer the interested readers to~\citet{raydan1993barzilai,raydan1997barzilai,dai2002r,fletcher2005barzilai} and references therein. BB method has been successfully applied for solving problems arising from emerging applications, such as compressed sensing~\citep{wright2009sparse}, sparse reconstruction~\citep{wen2010fast} and image processing~\citep{wang2007projected}.

\section{Proofs of Theorems}\label{supp:proofs}
\subsection{Proof of Theorem~\ref{thm:statistical}}\label{supp:thm_statistical}

We first state the regularity conditions needed for establishing convergence of the SG-PALM estimators $\{\hat{\mat\Psi}_k\}_{k=1}^K$ to their true value $\{\bar{\mat\Psi}_k\}_{k=1}^K$.

\noindent \textbf{(A1 - Subgaussianity)} The data $\tensor{X}^1,\dots,\tensor{X}^N$ are i.i.d subgaussian random tensors, that is, $\vecto(\tensor{X}^i) \sim \mat{x}$, where $\mat{x}$ is a subgaussian random vector in $\mathbb{R}^d$, i.e., there exist a constant $c>0$, such that for every $\mat{a} \in \mathbb{R}^d$, $\mathbb{E}e^{\mat{a}^T x} \leq e^{c\mat{a}^T \bar{\mat{\Sigma}} \mat{a}}$, and there exist $\rho_j > 0$ such that $\mathbb{E}e^{tx_j^2} \leq +\infty$ whenever $|t| < \rho_j$, for $1 \leq j \leq d$.

\noindent \textbf{(A2 - Bounded eigenvalues)} There exist constants $0 < \Lambda_{\min} \leq \Lambda_{\max} < \infty$, such that the minimum and maximum eigenvalues of $\mat{\Omega}$ are bounded with $\lambda_{\min}(\bar{\mat{\Omega}}) = (\sum_{k=1}^K \lambda_{\max}(\mat{\Psi}_k))^{-2} \geq \Lambda_{\min}$ and $\lambda_{\max}(\bar{\mat{\Omega}}) = (\sum_{k=1}^K \lambda_{\min}(\mat{\Psi}_k))^{-2} \leq \Lambda_{\max}$.

\noindent \textbf{(A3 - Incoherence condition)} There exists a constant $\delta < 1$ such that for $k=1,\dots,K$ and all $(i,j) \in \mathcal{A}_{k}$
\begin{equation*}
    |\bar{\mathcal{L}}_{ij,\mathcal{A}_{k}}^{''}(\bar{\mat{\Psi}})[\bar{\mathcal{L}}_{\mathcal{A}_{k},\mathcal{A}_{k}}^{''}(\bar{\mat{\Psi}})]^{-1} \text{sign}(\bar{\mat{\Psi}}_{\mathcal{A}_{k},\mathcal{A}_{k}})| \leq \delta,
\end{equation*} where for each $k$ and $1 \leq i < j \leq d_k$, $1 \leq k < l \leq d_k$,
\begin{equation*}
    \bar{\mathcal{L}}_{ij,kl}^{''}(\bar{\mat{\Psi}}) := E_{\bar{\mat{\Psi}}} \Bigg(\frac{\partial^2 \mathcal{L}(\mat{\Psi})}{\partial(\mat{\Psi}_k)_{i,j} \partial(\mat{\Psi}_k)_{k,l}}|_{\mat{\Psi}=\bar{\mat{\Psi}}} \Bigg),
\end{equation*}
and
\begin{equation*}
    \mathcal{L}(\mat\Psi)
    = -\frac{N}{2} \log | (\bigoplus_{k=1}^K \diag(\mat\Psi_k))^2|  + \frac{N}{2} \tr(\mat{S} \cdot (\bigoplus_{k=1}^K \mat\Psi_k)^2).
\end{equation*}

Given assumptions (A1-A3), the theorem follows from Theorem 3.3 in \citet{wang20sylvester}.

\subsection{Proof of Theorem~\ref{thm:sg-palm-main}}\label{supp:thm_sg-palm-main}
We next turn to convergence of the iterates $\{\mat\Psi^t\}$ from SG-PALM to a global optimum of \eqref{eqn:objective}. The proof leverages recent results in the convergence of alternating minimization algorithms for non-strongly convex objective~\citep{bolte2014proximal,karimi2016linear,li2018calculus,zhang2020new}. We outline the proof strategy:
\begin{enumerate}
    \item We establish Lipschitz continuity of the blockwise gradient $\nabla_kH(\mat\Psi)$ for $k=1,\dots,K$.
    \item We show that the objective function $\mathcal{L}_{\mat\lambda}$ satisfies the Kurdyka - \L ojasiewicz (KL) property. Further, it has a KL exponent of $\frac{1}{2}$ (defined later in the proofs).
    \item The KL property (with exponent $\frac{1}{2}$) is equivalent to a generalized Error Bound (EB) condition, which enables us to establish linear iterative convergence of the objective function~\eqref{eqn:objective} to its global optimum.
\end{enumerate}

\begin{definition}[Subdifferentials]\label{def:subdiff}
Let $f: \reals^d \rightarrow (-\infty,+\infty]$ be a proper and lower semicontinuous function. Its domain is defined by
\begin{equation*}
    \text{dom}f := \{x \in \reals^d: f(x) < + \infty\}.
\end{equation*}
If we further assume that $f$ is convex, then the subdifferential of $f$ at $x \in \text{dom}f$ can be defined by
\begin{equation*}
    \partial f(x) := \{v \in \reals^d: f(z) \geq f(x) + <v,z-x>, \forall z \in \reals^d\}.
\end{equation*}
The elements of $\partial f(x)$ are called subgradients of $f$ at $x$.
\end{definition}

Denote the domain of $\partial f$ by $\text{dom}\partial f := \{x \in \reals^d: \partial f(x) \neq \emptyset\}$. Then, if $f$ is proper, semicontinuous, convex, and $x \in \text{dom}f$, then $\partial f(x)$ is a nonempty closed convex set. In this case, we denote by $\partial^0 f(x)$ the unique least-norm element of $\partial f(x)$ for $x \in \text{dom}\partial f$, along with $\|\partial^0 f(x)\|=+\infty$ for $x \notin \text{dom}\partial f$. Points whose subdifferential contains $0$ are critical points, denoted by $\textbf{crit}f$. For convex $f$, $\textbf{crit}f=\argmin f$.

\begin{definition}[KL property]\label{def:kl}
Let $\Gamma_{c_2}$ stands for the class of functions $\phi:[0,c_2] \rightarrow \mathbb{R}_{+}$ for $c_2>0$ with the properties:
\begin{enumerate}[i.]
    \item $\phi$ is continuous on $[0,c_2]$;
    \item $\phi$ is smooth concave on $(0,c_2)$;
    \item $\phi(0)=0, \phi'(s)>0, \forall s \in (0,c_2)$.
\end{enumerate}
Further, for $x \in \reals^d$ and any nonempty $Q \subset \reals^d$, define the distance function $d(x,Q):=\inf_{y \in Q} \|x-y\|$. Then, a function $f$ is said to have the Kurdyka - \L ojasiewicz (KL) property at point $x_0$, if there exist $c_1 > 0$, a neighborhood $B$ of $x_0$, and $\phi \in \Gamma_{c_2}$ such that for all
\begin{equation*}
    x \in B(x_0,c_1) \cap \{x:f(x_0)<f(x)<f(x_0)+c_2\},
\end{equation*}
the following inequality holds
\begin{equation*}
    \phi'\Big(f(x)-f(x_0)\Big)\text{dist}(0,\partial f(x)) \geq 1.
\end{equation*}
If $f$ satisfies the KL property at each point of $\text{dom}\partial f$ then $f$ is called a KL function.
\end{definition}

We first present two lemmas that characterize key properties of the loss function.

\begin{lemma}[Blockwise Lipschitzness]\label{lemma:lip}
The function $H$ is convex and continuously differentiable on an open set containing $\text{dom} G$ and its gradient, is block-wise Lipschitz continuous with block Lipschitz constant $L_k>0$ for each $k$, namely for all $k=1,\dots,K$ and all $\mat\Psi_k, \mat\Psi_k' \in \mathbb{R}^{d_k \times d_k}$
\begin{equation*}
\begin{aligned}
    & \|\nabla_k H(\mat\Psi_{i < k},\mat\Psi_k,\mat\Psi_{i > k}) - \nabla_k H(\mat\Psi_{i < k},\mat\Psi_k',\mat\Psi_{i > k})\| \\
    & \leq L_k \|\mat\Psi_k - \mat\Psi_k'\|,
\end{aligned}
\end{equation*}
where $\nabla_k H$ denotes the gradient of $H$ with respect to $\mat\Psi_k$ with all remaining $\mat\Psi_i$, $i \neq k$ fixed. Further, the function $G_k$ for each $k=1,\dots,K$ is a proper lower semicontinuous (lsc) convex function.
\end{lemma}
\begin{proof}
For simplicity of notation, in this and the following proofs we use $\mat\Psi$ (i.e., omitting the subscript) to denote the set $\{\mat\Psi_k\}_{k=1}^K$ or the $K$-tuple $(\mat\Psi_1,\dots,\mat\Psi_K)$ whenever there is no confusion. Recall the blockwise gradient of the smooth part of the objective function $H$ with respect to $\mat\Psi_k$, for each $k=1,\dots,K$, is given by
\begin{equation*}
\begin{aligned}
    \nabla_k H(\mat\Psi) &= \diag\Big(\Big[\tr\{(\diag((\mat\Psi_k))_{ii} + \bigoplus_{j \neq k}\diag(\mat\Psi_j))^{-1}\} \quad i=1:d_k \Big]\Big) \\
    & \quad + \mat{S}_k\mat\Psi_k + \mat\Psi_k\mat{S}_k + 2\sum_{j \neq k}\mat{S}_{j,k}.
\end{aligned}
\end{equation*}
Then for $\mat\Psi_k,\mat\Psi'_k$, 
\begin{equation*}
\begin{aligned}
    & \|\mat{S}_k\mat\Psi_k + \mat\Psi_k\mat{S}_k + 2\sum_{j \neq k}\mat{S}_{j,k} - (\mat{S}_k\mat\Psi'_k + \mat\Psi'_k\mat{S}_k + 2\sum_{j \neq k}\mat{S}_{j,k})\| \\
    & = \|\mat{S}_k\mat\Psi_k + \mat\Psi_k\mat{S}_k - \mat{S}_k\mat\Psi'_k - \mat\Psi'_k\mat{S}_k\| \\
    & \leq 2\|\mat{S}_k\| \|\mat\Psi_k - \mat\Psi'_k\|.
\end{aligned}
\end{equation*}
To prove Lipschitzness of the remaining parts, we consider the case of $K=2$ for simplicity of notations. The arguments easily carry over cases of $K>2$. In this case, denote $\mat{A}=(a_{ij}):=\mat\Psi_1$ and $\mat{B}=(b_{kl}):=\mat\Psi_2$. Let $f(\mat{A}):=\frac{\partial}{\partial \mat{A}}\log|\diag(\mat{A} \oplus \mat{B})|$, then
\begin{equation*}
    f(\mat{A}) - f(\mat{A}') = \diag \Big(\Big[ \sum_{i=1}^{m_2}(a_{jj}+b_{ii})^{-1} - \sum_{i=1}^{m_2}(a'_{jj}+b_{ii})^{-1} \quad j=1,\dots,m_1 \Big] \Big)
\end{equation*}
and
\begin{equation*}
    \begin{aligned}
        \|f(\mat{A}) - f(\mat{A}')\|_F &= \Bigg(\sum_{j=1}^{m_1}\Big(\sum_{i=1}^{m_2}(a_{jj}+b_{ii})^{-1} - \sum_{i=1}^{m_2}(a'_{jj}+b_{ii})^{-1}\Big)^2\Bigg)^{1/2} \\
        &\leq \Bigg(\sum_{j=1}^{m_1}m_2\sum_{i=1}^{m_2}\Big((a_{jj}+b_{ii})^{-1} - (a'_{jj}+b_{ii})^{-1}\Big)^2\Bigg)^{1/2} \\
        &= \Big(m_2\sum_{j=1}^{m_1}\sum_{i=1}^{m_2}(c_{ji}^{-1}-(c'_{ji})^{-1})^2\Big)^{1/2} \\
        &= \Big(m_2\sum_{j=1}^{m_1}\sum_{i=1}^{m_2}(c'_{ji})^{-2}(c'_{ji}-c_{ji})^2c_{ji}^{-2}\Big)^{1/2} \\
        &= \Big(m_2\sum_{j=1}^{m_1}(a_{jj}-a'_{jj})^2\sum_{i=1}^{m_2}(c'_{ji}c_{ji})^{-2}\Big)^{1/2} \\
        &\leq \Big(Cm_2\sum_{j=1}^{m_1}\sum_{i=1}^{m_2}(c'_{ji}c_{ji})^{-2}\Big)^{1/2} \|\mat{A} - \mat{A}'\|_F,
    \end{aligned}
\end{equation*}
where the first inequality is due to Cauchy-Schwartz inequality; the third line is due to $c_{ji}:=a_{jj}+b_{ii}$; and in the last inequality we upper-bound each $(a_{jj}-a'_{jj})^2$ by its maximum over all $j$, which is absorbed in  a constant $C$. Note that the first term in the last line above is finite as long as the summations of the diagonal elements of the factors $\mat{A}$ and $\mat{B}$ are finite, which is implied if the precision matrix $\mat{\Omega}$ defined by the Sylvester generating equation as $(\mat{A} \oplus \mat{B})^2$ has finite diagonal elements. This follows from Theorem 3.1 of \citet{oh2014optimization}, who proved that if a symmetric matrix $\mat\Omega$ satisfying $\mat\Omega \in \mathcal{C}_0$, where
\begin{equation*}
    \mathcal{C}_0 = \Big\{\mat\Omega| \mathcal{L}_{\mat\lambda}(\mat\Omega) \leq \mathcal{L}_{\mat\lambda}(\mat\Omega^{(0)})=M \Big\},
\end{equation*}
and $\mat\Omega^{(0)}$ is an arbitrary initial point with a finite function value $\mathcal{L}_{\mat\lambda}(\mat\Omega^{(0)}):=M$, the diagonal elements of $\mat\Omega$ are bounded above and below by constants which depend only on $M$, the regularization parameter $\mat\lambda$, and the sample covariance matrix $\mat{S}$. Therefore, we have
\begin{equation*}
    \|f(\mat{A}) - f(\mat{A}')\|_F \leq \Tilde{C}\|\mat{A}-\mat{A}'\|_F
\end{equation*}
for some constant $\Tilde{C} \in (0,+\infty)$. Similarly, we can establish such an inequality for $\mat{B}$, proving that the first term in $\nabla_k H$ is Lipschitz continuous.
\end{proof}

As a consequence of Lemma~\ref{lemma:lip}, the gradient of $H$, defined by $\nabla H = (\nabla_1 H,\dots,\nabla_K H)$ is Lipschitz continuous on bounded subsets $\mathbb{B}_1 \times \cdots \times \mathbb{B}_K$ of $\mathbb{R}^{d_1 \times d_1} \times \cdots \times \mathbb{R}^{d_K \times d_K}$ with some constant $L>0$, such that for all $(\mat\Psi_k,\mat\Psi_k') \in \mathbb{B}_k \times \mathbb{B}_k$,
\begin{equation*}
\begin{aligned}
    & \|(\nabla_1 H(\mat\Psi_1,\mat\Psi_{i > 1})-\nabla_1 H(\mat\Psi_1',\mat\Psi_{i > 1}'),\dots,\\
    & \nabla_K H(\mat\Psi_{i < K}',\mat\Psi_K')-\nabla_K H(\mat\Psi_{i < K}',\mat\Psi_K'))\| \\ 
    &\leq L\|(\mat\Psi_1-\mat\Psi_1',\dots,\mat\Psi_K-\mat\Psi_K')\|,
\end{aligned}
\end{equation*}
and we have $L \leq \sum_{k=1}^K L_k$.

\begin{lemma}[KL property of $\mathcal{L}_{\mat\lambda}$]\label{lemma:kl_loss}
The objective function $\mathcal{L}_{\mat\lambda}(\mat\Psi)$ defined in~\eqref{eqn:objective} satisfies the KL property. Further, $\phi$ in this case can be chosen to have the form $\phi(s) = \alpha s^{1/2}$, where $\alpha$ is some positive real number. Functions satisfying the KL property with this particular choice of $\phi$ is said to have a KL exponent of $\frac{1}{2}$.
\end{lemma}
\begin{proof}
This can be established in a few steps:
\begin{enumerate}
    \item It can be shown that the function (of $\mat{X}$) $\tr(\mat{S}\mat{X}^2) + \|\mat{X}\|_{1,\text{off}}$
    satisfies the KL property with exponent $\frac{1}{2}$~\citep{karimi2016linear}. We then apply the calculus rules of the KL exponent (compositions and separable summations) studied in~\citet{li2018calculus} to prove that $\tr(\mat{S}(\bigoplus_j\mat\Psi_j)^2)$ and $\sum_j\|\mat\Psi_j\|_{1,\text{off}}$ are also KL functions with exponent $\frac{1}{2}$. 
    \item The $-\log\det\Big(\bigoplus_j\diag(\mat\Psi_j)\Big)$ term can be shown to be KL with exponent $\frac{1}{2}$ using a transfer principle studied in~\citet{lourencco2019generalized}.
    \item Finally, using the calculus rules of KL exponent one more time, we combine the first two results and establish that $\mathcal{L}_{\mat\lambda}$ has KL exponent of $\frac{1}{2}$.
\end{enumerate}

 \citet{karimi2016linear} proved that the following function, parameterized by some symmetric matrix $\mat{X}$, satisfies the KL property with KL exponent $\frac{1}{2}$:
\begin{equation*}
    \tr(\mat{S}\mat{X}^2) + \|\mat{X}\|_{1,\text{off}} = \|\mat{A}\mat{X}\|_F^2 + \|\mat{X}\|_{1,\text{off}}
\end{equation*}
for $\mat{S}=\mat{A}\mat{A}^T$, even when $\mat{A}$ is not of full rank.   

We apply the calculus rules of the KL exponent studied in~\citet{li2018calculus} to prove that $\tr(\mat{S}(\bigoplus_j\mat\Psi_j)^2)$ and $\sum_j\|\mat\Psi_j\|_{1,\text{off}}$ are KL functions with exponent $\frac{1}{2}$. 
Particularly, we observe that $\tr\Big(\mat{S}(\bigoplus_j \mat\Psi_j)^2\Big)$ is the composition of functions $\mat{X} \rightarrow \tr(\mat{S}\mat{X})$ and $(\mat{X}_1,\dots,\mat{X}_K) \rightarrow \bigoplus_j \mat{X}_j$; and $\sum_j\|\mat\Psi_j\|_{1,\text{off}}$ is a separable block summation of functions $\mat{X}_j \rightarrow \|\mat{X}_j\|_{1,\text{off}}$. 

Thus, by Theorem 3.2. (exponent for composition of KL functions) in \citet{li2018calculus}, since the Kronecker sum operation is linear and hence continuously differentiable, the trace function is KL with exponent $\frac{1}{2}$, and the mapping $(\mat{X}_1,\dots,\mat{X}_K) \rightarrow \bigoplus_j \mat{X}_j$ is clearly one to one, the function $\tr(\mat{S}(\bigoplus_j\mat\Psi_j)^2)$ has the KL exponent of $\frac{1}{2}$. By Theorem 3.3. (exponent for block separable sums of KL functions) in \citet{li2018calculus}, since the function $\|\cdot\|_{1,\text{off}}$ is proper, closed, continuous on its domain, and is KL with exponent $\frac{1}{2}$, the function $\|\mat{X}_j\|_{1,\text{off}}$ is KL with an exponent of $\frac{1}{2}$.

It remains to prove that $-\log\det\Big(\bigoplus_j\diag(\mat\Psi_j)\Big)$ is also a KL function with an exponent of $\frac{1}{2}$. By Theorem 30 in \citet{lourencco2019generalized}, if we have $f:\mathbb{R}^r \rightarrow \mathbb{R}$ a symmetric function and $F:\mathcal{E} \rightarrow \mathbb{R}$ the corresponding spectral function, the followings hold
\begin{enumerate}[(i).]
    \item $F$ satisfies the KL property at $\mat{X}$ iff $f$ satisfies the KL property at $\lambda(\mat{X})$, i.e., the eigenvalues of $\mat{X}$.
    \item $F$ satisfies the KL property with exponent $\alpha$ iff $f$ satisfies the KL property with exponent $\alpha$ at $\lambda(\mat{X})$.
\end{enumerate}
Here, take $f(\lambda(\mat{X})):=-\sum_{i=1}^r\log(\lambda_i(\mat{X}))$, and $F(\mat{X}):=-\log\det(\mat{X})$ the corresponding spectral function. Then, the function $f$ is symmetric since its value is invariant to permutations of its arguments, and it is a strictly convex function in its domain, so it satisfies the KL property with an exponent of $\frac{1}{2}$. Therefore, $F$ satisfies the KL property with the same KL exponent of $\frac{1}{2}$. Now, we apply the calculus rules for KL functions again. As both the Kronecker sum and the $\diag$ operators are linear, we conclude that $-\log\det\Big(\bigoplus_j\diag(\mat\Psi_j)\Big)$ is a KL function with an exponent of $\frac{1}{2}$.

Overall, we have that the negative log-pseudolikelihood function $\mathcal{L}(\mat\Psi)$ satisfies the KL property with an exponent of $\frac{1}{2}$.
\end{proof}

Now we are ready to prove Theorem~\ref{thm:sg-palm-main}. We follow \citet{zhang2020new} and divide the proof into three steps.

\paragraph{Step 1.} We obtain a sufficient decrease property for the loss function $\mathcal{L}$ in terms of the squared distance of two successive iterates:
\begin{equation}\label{eqn:sufficient_descent}
    \mathcal{L}(\mat\Psi^{(t)}) - \mathcal{L}(\mat\Psi^{(t+1)}) \geq \frac{L_{\min}}{2} \|\mat\Psi^{(t)} - \mat\Psi^{(t+1)}\|^2.
\end{equation}

Here and below, $\mat\Psi^{(t+1)}:=(\mat\Psi_1^{(t+1)},\dots,\mat\Psi_K^{(t+1)})$ and $L_{\min}:=\min_k L_k$. First note that at iteration $t$, the line search condition is satisfied for step size $\frac{1}{\eta_k^{(t)}} \geq L_k$, where $L_k$ is the Lipschitz constant for $\nabla_k H$. Further, it follows that for SG-PALM with backtracking one has for every $t \geq 0$ and each $k=1,\dots,K$,
\begin{equation*}
    \frac{1}{\eta_k^{(0)}} \leq \frac{1}{\eta_k^{(t)}} \leq c L_k,
\end{equation*}
where $c>0$ is the backtracking constant.

Then by Lemma 3.1 in \citet{shefi2016rate}, we get 
\begin{equation*}
    \begin{aligned}
        \mathcal{L}(\mat\Psi^{(t)}) - \mathcal{L}(\mat\Psi^{(t+1)}) &\geq \frac{1}{2\eta_{\min}^{(t+1)}} \|\mat\Psi^{(t)} - \mat\Psi^{(t+1)}\|^2 \\
        &\geq  \frac{L_{\min}}{2} \|\mat\Psi^{(t)} - \mat\Psi^{(t+1)}\|^2
    \end{aligned}
\end{equation*}
for $\eta_{\min}^{(t)}:=\min_k \eta_{k}^{(t)}$.

\paragraph{Step 2.} By Lemma~\ref{lemma:kl_loss}, $\mathcal{L}$ satisfies the KL property with an exponent of $\frac{1}{2}$. Then from Definition~\ref{def:kl}, this suggests that at $x=\mat\Psi^{t+1}$ and $f(x_0)=\min\mathcal{L}$
\begin{equation}\label{eqn:res-obj-EB}
    \|\partial^0\mathcal{L}(\mat\Psi^{t+1})\| \geq \alpha \sqrt{\mathcal{L}(\mat\Psi^{t+1}) - \min\mathcal{L}},
\end{equation}
where $\alpha>0$ is a fixed constant defined in Lemma~\ref{lemma:kl_loss}. This property is equivalent to the error bound condition, ($\partial^0 \mathcal{L}, \alpha, \Omega$)-(res-obj-EB), defined in Definition 5 in \citet{zhang2020new}, for $\Omega \subset \text{dom}\partial\mathcal{L}$. This is strictly weaker than strong convexity (see Section 4 in \citet{zhang2020new}).

At iteration $t+1$, there exists $\xi_k^{(t+1)} \in \partial G_k(\mat\Psi_k^{(t+1)})$ satisfying the optimality condition:
\begin{equation*}
    \nabla_k H(\mat\Psi_{i<k}^{(t+1)},\mat\Psi_{i \geq k}^{(t)}) + \frac{1}{\eta_k^{(t+1)}}(\mat\Psi_k^{(t+1)} - \mat\Psi_k^{(t)}) + \xi_k^{(t+1)} = 0.
\end{equation*}
Let $\xi^{(t+1)}:=(\xi_1^{(t+1)},\dots,\xi_K^{(t+1)})$. Then,
\begin{equation*}
    \nabla H(\mat\Psi^{(t+1)}) + \xi^{(t+1)} \in \partial\mathcal{L}(\mat\Psi^{(t+1)})
\end{equation*}
and hence the error bound condition becomes
\begin{equation*}
    \mathcal{L}(\mat\Psi^{(t+1)}) - \min\mathcal{L} \leq \frac{\|\partial^0\mathcal{L}(\mat\Psi^{(t+1)})\|^2}{\alpha^2} \leq \frac{\|\nabla H(\mat\Psi^{(t+1)}) + \xi^{(t+1)}\|^2}{\alpha^2}.
\end{equation*}
It follows that
\begin{equation*}
    \begin{aligned}
        \|\nabla H(\mat\Psi^{(t+1)}) + \xi^{(t+1)}\|^2 &= \sum_{k=1}^K \|\nabla_k H(\mat\Psi^{(t+1)}) -\nabla_k H(\mat\Psi_{i<k}^{(t+1)},\mat\Psi_{i \geq k}^{(t)}) - \frac{1}{\eta_k^{(t+1)}}(\mat\Psi_k^{(t+1)} - \mat\Psi_k^{(t)})\|^2 \\
        &\leq \sum_{k=1}^K 2\|\nabla_k H(\mat\Psi^{(t+1)}) -\nabla_k H(\mat\Psi_{i<k}^{(t+1)},\mat\Psi_{i \geq k}^{(t)})\|^2 + \sum_{k=1}^K \frac{2}{(\eta_k^{(t+1)})^2}\|\mat\Psi_k^{(t+1)} - \mat\Psi_k^{(t)}\|^2 \\
        &\leq \sum_{k=1}^K 2\|\nabla H(\mat\Psi^{(t+1)}) -\nabla H(\mat\Psi_{i<k}^{(t+1)},\mat\Psi_{i \geq k}^{(t)})\|^2 + \sum_{k=1}^K \frac{2}{(\eta_k^{(t+1)})^2}\|\mat\Psi_k^{(t+1)} - \mat\Psi_k^{(t)}\|^2 \\
        &\leq \sum_{k=1}^K 2\Big(\sum_{j=1}^K\frac{1}{\eta_j^{(t+1)}}\Big)^2 \|\mat\Psi^{(t+1)}_{i \geq k} - \mat\Psi^{(t)}_{i \geq k}\|^2 + \sum_{k=1}^K \frac{2}{(\eta_k^{(t+1)})^2}\|\mat\Psi_k^{(t+1)} - \mat\Psi_k^{(t)}\|^2 \\
        &\leq \Bigg(2Kc^2\Big(\sum_{j=1}^K L_j\Big)^2 + 2c^2L_{\max}\Bigg)\|\mat\Psi^{(t+1)} - \mat\Psi^{(t)}\|^2.
    \end{aligned}
\end{equation*}
Therefore, we get
\begin{equation}\label{eqn:obj_gap}
    \mathcal{L}(\mat\Psi^{(t+1)}) - \min\mathcal{L} \leq \frac{\Bigg(2Kc^2\Big(\sum_{j=1}^K L_j\Big)^2 + 2c^2L_{\max}\Bigg)}{\alpha^2} \|\mat\Psi^{(t+1)} - \mat\Psi^{(t)}\|^2.
\end{equation}

\paragraph{Step 3.} Combining \eqref{eqn:sufficient_descent} and \eqref{eqn:obj_gap}, we have
\begin{equation*}
    \begin{aligned}
        \mathcal{L}(\mat\Psi^{(t)}) - \min\mathcal{L} &= \Big(\mathcal{L}(\mat\Psi^{(t)}) - \mathcal{L}(\mat\Psi^{(t+1)})\Big) + \Big(\mathcal{L}(\mat\Psi^{(t+1)}) - \min\mathcal{L}\Big) \\
        &\geq \frac{L_{\min}}{2} \|\mat\Psi^{(t)} - \mat\Psi^{(t+1)}\|^2 + \Big(\mathcal{L}(\mat\Psi^{(t+1)}) - \min\mathcal{L}\Big) \\
        &\geq \Bigg(\frac{\alpha^2L_{\min}}{4Kc^2(\sum_{j=1}^K L_j)^2 + 4c^2L_{\max}} + 1\Bigg) \Big(\mathcal{L}(\mat\Psi^{(t+1)}) - \min\mathcal{L}\Big). 
    \end{aligned}
\end{equation*}
This completes the proof.

\section{SG-PALM with Non-Convex Regularizers}\label{supp:nonconvex}
The estimation algorithm for non-convex regularizer is largely the same as Algorithm~\ref{alg:sg-palm}, except with an additional term added to the gradient term. Specifically, the updates are of the form
\begin{equation*}
    \mat\Psi_k^{(t+1)} = \text{prox}_{\eta^t_k\lambda_k}^{\|\cdot\|_{1,\text{off}}}\Big(\mat\Psi_k^t - \eta^t_k \nabla_k \bar{H}(\mat\Psi_{i < k}^{t+1},\mat\Psi_{i \geq k}^t)\Big),
\end{equation*}
where 
\begin{equation*}
    \bar{H}(\mat\Psi) = H(\mat\Psi) + \sum_{k=1}^K \sum_{i \neq j} \Big(g_{\lambda_k}([\mat\Psi_k]_{i,j})-\lambda_k|[\mat\Psi_k]_{i,j}|\Big).
\end{equation*}
Here, the formulation covers a range of non-convex regularizations. Particularly, the SCAD penalty~\citep{fan2001variable} with parameter $a>2$ is given by
\begin{equation*}
    g_\lambda(t) =
    \begin{cases}
    \lambda|t|, \quad \text{if} \quad |t|<\lambda \\
    -\frac{t^2-2a\lambda|t|+\lambda^2}{2(a-1)}, \quad \text{if} \quad \lambda < |t| < a\lambda \\
    \frac{(a+1)\lambda^2}{2}, \quad \text{if} \quad a\lambda < |t|,
    \end{cases}
\end{equation*}
which is linear for small $|t|$, constant for large $|t|$, and a transition between the two regimes for moderate $|t|$. 

The MCP penalty~\citep{zhang2010nearly} with parameter $a>0$ is given by
\begin{equation*}
    g_\lambda(t) = \text{sign}(t)\lambda\int_0^{|t|}\Big(1-\frac{z}{a\lambda}\Big)_+dz,
\end{equation*}
which gives a smoother transition between the approximately linear region and the constant region ($t>a\lambda$) as defined in SCAD.

The updates can also be written as
\begin{equation*}
    \mat\Psi_k^{(t+1)} = \text{prox}_{\eta^t_k\lambda_k}^{\|\cdot\|_{1,\text{off}}}\Bigg(\mat\Psi_k^t - \eta^t_k \nabla_k \Big(H(\mat\Psi_{i < k}^{t+1},\mat\Psi_{i \geq k}^t) + Q_{\lambda_k}'(\mat\Psi_k)\Big)\Bigg),
\end{equation*}
where $q_{\lambda}'(t):=\frac{d}{dt}(g_\lambda(t)-\lambda|t|)$ for $t \neq 0$ and $q_{\lambda}'(0)=0$ and $Q_\lambda'$ denotes $q_\lambda'$ applied elementwise to a matrix argument. These updates can be inserted into the framework of Algorithm~\ref{alg:sg-palm}. The details are summarized in Algorithm~\ref{alg:sg-palm-noncvx}.

\begin{algorithm}[!tbh]
\begin{algorithmic}
\caption{SG-PALM with non-convex regularizer}\label{alg:sg-palm-noncvx}
\REQUIRE Data tensor $\tensor{X}$, mode-$k$ Gram matrix $\mat{S}_k$, regularizing parameter $\lambda_k$, backtracking constant $c \in (0,1)$, initial step size $\eta_0$, initial iterate $\mat\Psi_k$ for each $k=1,\dots,K$.
\WHILE{not converged}
    \FOR{$k=1,\dots,K$}
        \STATE \textit{Line search:} Let $\eta^t_k$ be the largest element of $\{c^j \eta_{k,0}^t\}_{j=1,\dots}$ such that condition~\eqref{eqn:linesearch-cond} is satisfied for $\mat\Psi_k^{t+1} = \text{prox}_{\eta^t_k\lambda_k}^{\|\cdot\|_{1,\text{off}}}\Bigg(\mat\Psi_k^t - \eta^t_k \nabla_k \Big(H(\mat\Psi_{i < k}^{t+1},\mat\Psi_{i \geq k}^t) + Q_{\lambda_k}'(\mat\Psi_k)\Big)\Bigg)$.
        \STATE \textit{Update:} $\mat\Psi_k^{t+1} \longleftarrow \text{prox}_{\eta^t_k\lambda_k}^{\|\cdot\|_{1,\text{off}}}\Bigg(\mat\Psi_k^t - \eta^t_k \nabla_k \Big(H(\mat\Psi_{i < k}^{t+1},\mat\Psi_{i \geq k}^t) + Q_{\lambda_k}'(\mat\Psi_k)\Big)\Bigg)$.
    \ENDFOR
    \STATE \textit{Next initial step size:} Compute Barzilai-Borwein step size $\eta_0^{t+1}=\min_k \eta^{t+1}_{k,0}$, where $\eta^{t+1}_{k,0}$ is computed via~\eqref{eqn:bb-step}.
\ENDWHILE
\ENSURE Final iterates $\{\mat\Psi_k\}_{k=1}^K$.
\end{algorithmic}
\end{algorithm}

\subsection{Convergence Property}
Consider a sequence of iterate $\{\mat{x}^t\}_{t \in \mathbb{N}}$ generated by a generic PALM algorithm for minimizing some objective function $f$. Specifically, assume 
\begin{itemize}
    \item[] $(\mathcal{H}_1)$ $\inf f > -\infty$.
    \item[] $(\mathcal{H}_2)$ The restriction of the function to its domain is a continuous function.
    \item[] $(\mathcal{H}_3)$ The function satisfies the KL property.
\end{itemize}

Then, as in Theorem 2 of \citet{attouch2009convergence}, if this objective function satisfying $(\mathcal{H}_1),(\mathcal{H}_2),(\mathcal{H}_3)$ in addition satisfies the KL property with
\begin{equation*}
    \phi(s) = \alpha s^{1-\theta},
\end{equation*}
where $\alpha>0$ and $\theta \in (0,1]$. Then, for $\mat{x}^{\ast}$ some critical point of $f$, the following estimations hold
\begin{enumerate}[(i).]
    \item If $\theta=0$ then the sequence of iterates converges to $\mat{x}^{\ast}$ in a finite number of steps.
    \item If $\theta \in (0,\frac{1}{2}]$ then there exist $\omega>0$ and $\tau \in [0,1)$ such that $\|\mat{x}^t-\mat{x}^{\ast}\| \leq \omega \tau^t$.
    \item If $\theta \in (\frac{1}{2},1)$ then there exist $\omega>0$ such that $\|\mat{x}^t-\mat{x}^{\ast}\| \leq \omega t^{-\frac{1-\theta}{1\theta-1}}$.
\end{enumerate}

In the case of SG-PALM with non-convex regularizations, so long as the non-convex $\mathcal{L}$ satisfies the KL property with an exponent in $(0,\frac{1}{2}]$, the algorithm remains linearly convergent (to a critical point). We argue that this is true for SG-PALM with MCP or SCAD penalty. \citet{li2018calculus} showed that penalized least square problems with such penalty functions satisfy the KL property with an exponent of $\frac{1}{2}$. The proof strategy for the convex case can be easily adopted, incorporating the KL results for MCP and SCAD in \citet{li2018calculus}, to show that the new $\mathcal{L}$ still has KL exponent of $\frac{1}{2}$. Therefore, SG-PALM with MCP or SCAD penalty converges linearly in the sense outlined above.

\section{Additional Details of the Solar Flare Experiments}\label{supp:additional_experiments}
\subsection{HMI and AIA Data}
The Solar Dynamics Observatory (SDO)/Helioseismic and Magnetic Imager (HMI) data characterize solar variability including the Sun's interior and the various components of magnetic activity; the SDO/Atmospheric Imaging Assembly (AIA) data contain a set of measurements of the solar atmosphere spectrum at various wavelengths. In general, HMI produces data that is particularly useful in determining the mechanisms of solar variability and how the physical processes inside the Sun that are related to surface magnetic field and activity. AIA contains structural information about solar flares, and the the high AIA pixel values are correlated with the flaring intensities. We are interested in examining if combination of multiple instruments enhances our understanding of the solar flares, comparing to the case of single instrument. Both HMI and AIA produce multi-band (or multi-channel) images, for this experiment we use all three channels of the HMI images and $9.4, 13.1, 17.1, 19.3$ nm wavelength channels of the AIA images. For a detailed descriptions of the instruments and all channels of the images, see \url{https://en.wikipedia.org/wiki/Solar_Dynamics_Observatory} and the references therein. Furthermore, for training and testing involved in this study, we used the data described in~\citep{galvez2019machine}, which are further pre-processed HMI and AIA imaging data for machine learning methods.

\subsection{Classification of Solar Flares/Active Regions (AR)}
The classification system for solar flares uses the letters A, B, C, M or X, according to the peak flux in watts per square metre ($W/m^2$) of X-rays with wavelengths $100$ to $800$ picometres ($1$ to $8$ angstroms), as measured at the Earth by the GOES spacecraft (\url{https://en.wikipedia.org/wiki/Solar_flare#Classification}). Here, A usually refers to a ``quite'' region, which means that the peak flux of that region is not high enough to be classified as a real flare; B usually refers a ``weak'' region, where the flare is not strong enough to have impact on spacecrafts, earth, etc; and M or X refers to a ``strong'' region that is the most detrimental. Differentiating between a weak and a strong flare/region ahead of time is a fundamental task in space physics and has recently attracted attentions from the machine learning community~\citep{chen2019identifying,jiao2019solar,sun2019interpreting}. In our study, we also focus on B and M/X flares and attempt to predict the videos that lead to either one of these two types of flares.

\subsection{Run Time Comparison}
We compare run times of the SG-PALM algorithm for estimating the precision matrix from the solar flare data with SyGlasso. Table~\ref{tab:solar_flare_run_time} illustrates that the SG-PALM algorithm converges faster in wallclock time. Note that in this real dataset, which is potentially non-Gaussian, the convergence behavior of the algorithms is different compare to synthetic examples. Nonetheless, SG-PALM enjoys an order of magnitude speed-up over SyGlasso.

\begin{table}[!tbh]
\centering
\caption{Run time (in seconds) comparisons between SyGlasso and SG-PALM on solar flare data for different regularization parameters. Note that the SG-PALM is an order of magnitude faster that SyGlasso.}
\label{tab:solar_flare_run_time}
\begin{tabular}{|c||c|c|}
 \multicolumn{3}{c}{} \\
 \hline
 \multirow{2}{*}{$\lambda$} & \textbf{SyGlasso} & \textbf{SG-PALM} \\
 \cline{2-3} 
 & \textbf{iter} \quad \textbf{sec} & \textbf{iter} \quad \textbf{sec} \\
 \hline
 $0.28$ & $47$ \quad $5772.1$ & $89$ \quad $583.7$ \\
 $0.41$ & $43$ \quad $5589.0$ & $86$ \quad $583.4$\\
 $0.54$ & $45$ \quad $5673.7$ & $85$ \quad $568.8$ \\
 $0.67$ & $42$ \quad $5433.0$ & $77$ \quad $522.6$ \\
 $0.79$ & $39$ \quad $4983.2$ & $82$ \quad $511.4$ \\
 $0.92$ & $40$ \quad $5031.9$ & $72$ \quad $498.0$ \\
 $1.05$ & $39$ \quad $4303.7$ & $76$ \quad $452.2$ \\
 $1.18$ & $41$ \quad $4234.7$ & $64$ \quad $437.6$ \\
 $1.30$ & $40$ \quad $4039.5$ & $58$ \quad $406.9$ \\
 $1.43$ & $35$ \quad $3830.7$ & $64$ \quad $364.9$ \\
 \hline
\end{tabular}
\end{table}

\subsection{Examples of Predicted Magnetogram Images}
Figure~\ref{fig:hmi_predicted_vs_real_img} depicts examples of the predicted HMI channels by SG-PALM. We observe that the proposed method was able to reasonably capture various components of the magnetic field and activity. Note that the spatial behaviors of the HMI components are quite different from those of AIA channels, that is, the structures tend to be less smooth and continuous (e.g., separated holes and bright spots) in HMI.

\begin{figure}[tbh!] 
\centering
\begin{tabular}{@{}c@{}}
    Predicted HMI examples - B vs. M/X \\
    \rotatebox{90}{\qquad \qquad AR B}
    \includegraphics[width=0.7\textwidth]{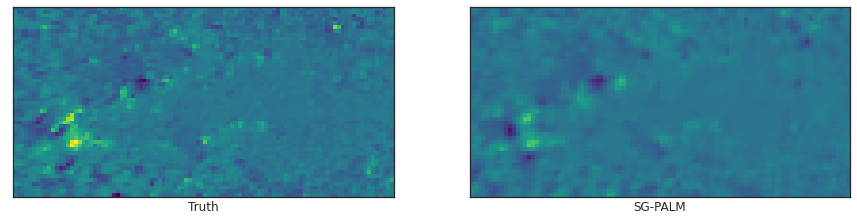}  \\
    \rotatebox{90}{\qquad \qquad AR B}
    \includegraphics[width=0.7\textwidth]{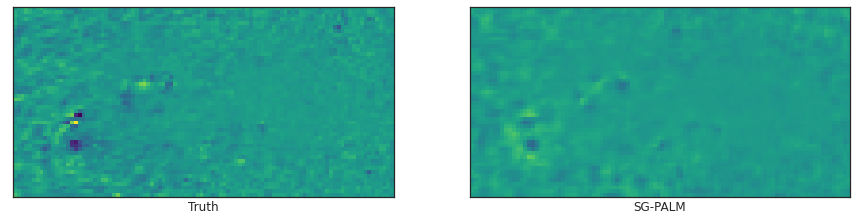} \\
    \rotatebox{90}{\qquad \qquad AR B}
    \includegraphics[width=0.7\textwidth]{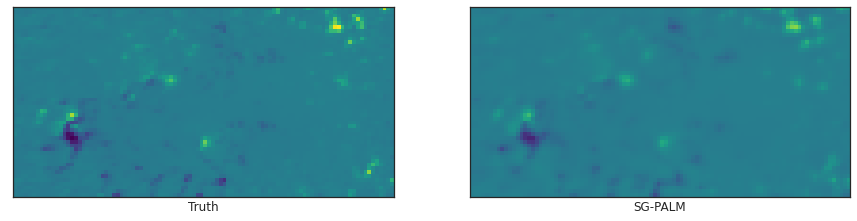} \\
    \rotatebox{90}{\qquad \qquad AR M/X}
    \includegraphics[width=0.7\textwidth]{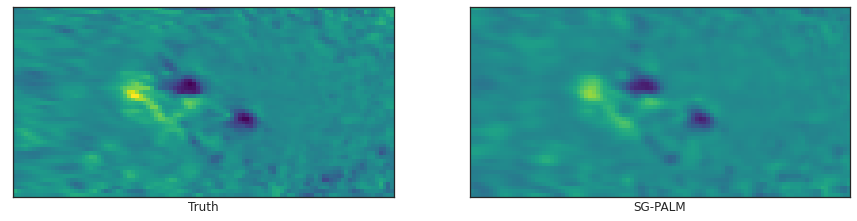} \\
    \rotatebox{90}{\qquad \qquad AR M/X}
    \includegraphics[width=0.7\textwidth]{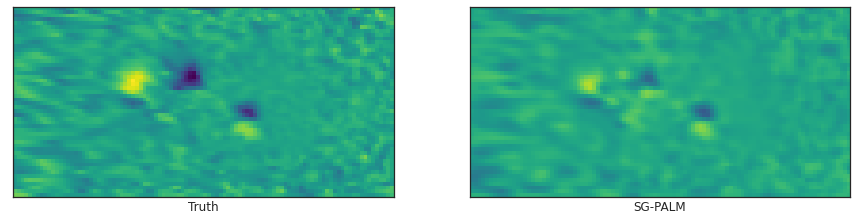} \\
    \rotatebox{90}{\qquad \qquad AR M/X}
    \includegraphics[width=0.7\textwidth]{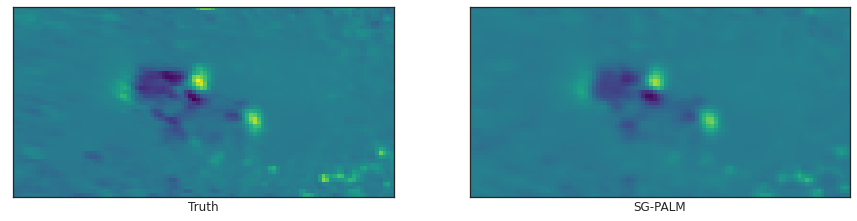}
\end{tabular}
\caption{Examples of one-hour ahead prediction of the first three channels (HMI components) of ending frames of $13$-frame videos, leading to B- (first three rows) and MX-class (last three rows) flares, produced by the SG-PALM, comparing to the real image (left column). Similarly to AIA predictions, linear forward predictors tend to underestimate the contrast ratio of the images. Nonetheless, the SG-PALM algorithm was able to both capture the spatial structures and intensities of the underlying magnetic fields. Note that the HMI images tend to be harder to predict, as indicated by the increased number and decreased degree of smoothness of features in the images, signifying the underlying magnetic activity on the solar surface.}
\label{fig:hmi_predicted_vs_real_img}
\end{figure}

\subsection{Multi-instrument vs. Single Instrument Prediction}
To illustrate the advantages of multi-instrument analysis, we compare the NRMSEs between an AIA-only (i.e., last four channels of the dataset) and an HMI\&AIA (i.e., all seven channels of the dataset) study in predicting the last frames of $13$-frame AIA videos, for each flare class, respectively, using the proposed SG-PALM. The results are depicted in Figure~\ref{fig:hmi_vs_aia}, where the average, standard deviation, and range of the NRMSEs across pixels are also shown for each error image. By leveraging the cross-instrument correlation structure, there is a $0.5\%-1\%$ drop in the averaged error rates and a $2\%-4\%$ drop in the range of the errors.

\begin{figure}[tbh!] 
\centering
\begin{tabular}{@{}c@{}}
    \qquad \qquad AIA Avg. NRMSE $=0.0379$ (with HMI channels) \quad AIA Avg. NRMSE $=0.0479$ (without HMI channels) \\
    \rotatebox{90}{\qquad \qquad AR B}
    \includegraphics[width=0.75\textwidth]{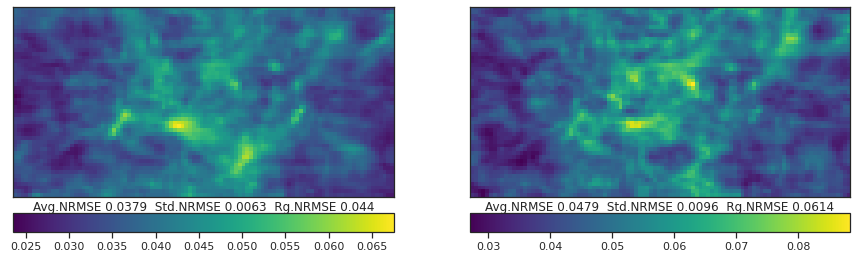}  \\
    \qquad \qquad AIA Avg. NRMSE $=0.0620$ (with HMI channels) \quad AIA Avg. NRMSE $=0.0674$ (without HMI channels) \\
    \rotatebox{90}{\qquad \qquad AR MX}
    \includegraphics[width=0.75\textwidth]{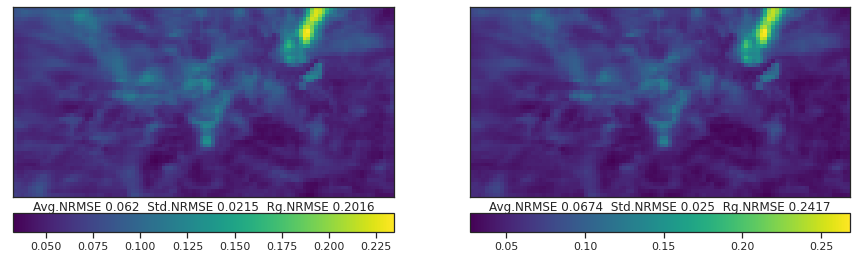} 
\end{tabular}
\caption{Comparison of the SG-PALM performance measured by NRMSE in predicting the AIA channels (i.e., last four channels) of the ending frame of $13$-frame videos leading to B- and MX-class solar flares, by using all HMI\&AIA channels (left column) and AIA-only channels (right column). The NRMSEs are computed by averaging across both testing samples and channels for each pixel. Note that there are improvements in both the averaged errors rates and the uncertainty in those errors (i.e., range of the errors) by including multi-instrument image channels.}
\label{fig:hmi_vs_aia}
\end{figure}

\subsection{Illustration of the Difficulty of Predictions for Two Flares Classes}
We demonstrate the difficulty of forward predictions of video frames. Figure~\ref{fig:mx_predicted_vs_prev_img} depicts two different channels of multiple frames from two videos leading to MX-class solar flares. Note that the current frame is the $13$th frame in the sequence that we are trying to predict. We observe that the prediction task is particularly difficult if there is a sudden transition of either the brightness or spatial structure of the frames near the end of the video. These sudden transitions are more frequent for MX flares than for B flares. In addition, as MX flares are generally considered as rare events (i.e., less frequent than B flares), it is harder for SG-PALM or related methods to learn a common correlation structures from training data.

On the other hand, typical image sequences leading to B flares exhibit much smoother transitions from frame to frame. As shown in Figure~\ref{fig:b_predicted_vs_prev_img}, the SG-PALM was able to produce remarkably good predictions of the current frames. 

\begin{figure}[tbh!] 
\centering
\begin{tabular}{@{}c@{}}
    Predicted examples - M/X \\
    \includegraphics[width=0.85\textwidth]{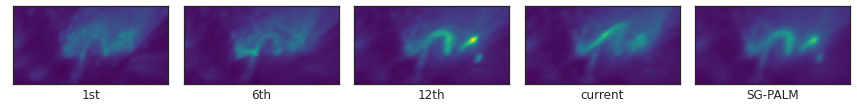}  \\
    \includegraphics[width=0.85\textwidth]{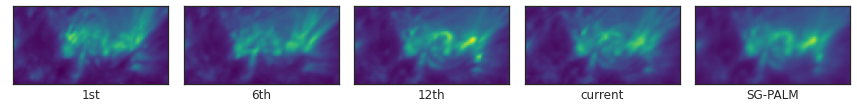} \\
    \includegraphics[width=0.85\textwidth]{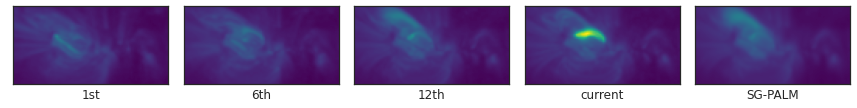} \\
    \includegraphics[width=0.85\textwidth]{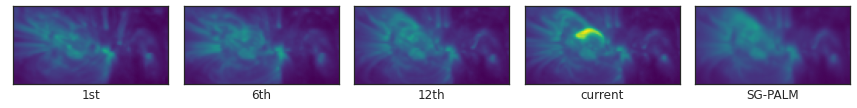}
\end{tabular}
\caption{Examples of frames at various timestamps of videos preceding the predictions of the last frames (last column) that lead to MX flares. Here, the first two rows correspond to the same video as the last two rows in Figure~\ref{fig:predicted_vs_real_img}. Note that the prediction tasks are difficult in these two extreme cases, where there are dramatic changes from the $12$th to the current ($13$th) frames.}
\label{fig:mx_predicted_vs_prev_img}
\end{figure}

\begin{figure}[tbh!] 
\centering
\begin{tabular}{@{}c@{}}
    Predicted examples - B  \\
    \includegraphics[width=0.85\textwidth]{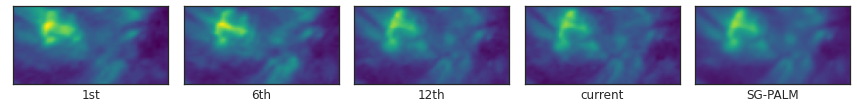} \\
    \includegraphics[width=0.85\textwidth]{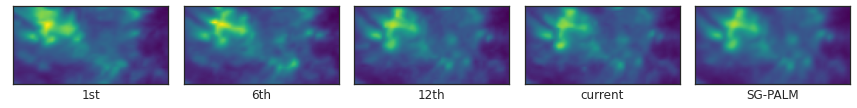} \\
    \includegraphics[width=0.85\textwidth]{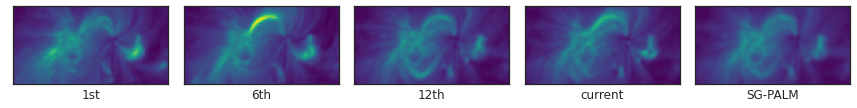} \\
    \includegraphics[width=0.85\textwidth]{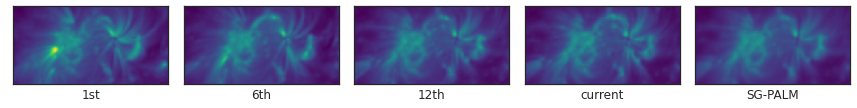}
\end{tabular}
\caption{Examples of frames at various timestamps of videos preceding the predictions of the last frames (last column) that lead to B flares. Here, the first two rows correspond to the same video as the first two rows in Figure~\ref{fig:predicted_vs_real_img}. Note that the prediction tasks are easier than those illustrated in Figure~\ref{fig:mx_predicted_vs_prev_img}, since the transitions near the end of the videos are much smoother.}
\label{fig:b_predicted_vs_prev_img}
\end{figure}

\subsection{Illustration of the Estimated Sylvester Generating Factors}
Figure~\ref{fig:psih_all} illustrates the patterns of the estimated Sylvester generating factors ($\mat\Psi_k$'s) for each flare class. Here, the videos from both classes appear to form Markov Random Fields, that is, each pixel only depends on its close neighbors in space and time given all other pixels. This is demonstrated by observing that the temporal or each of the spatial generating factor, which can be interpreted as conditional dependence graph for the corresponding mode, has its energies concentrate around the diagonal and decay as the nodes move far apart (in space or time).

The spatial patterns are similar for different flares. Although the exact spatial patterns are different from one frame to another, they always have their energies being concentrated at certain region (i.e., the brightest spot) that is usually close to the center of the images. This is due to the way how these images were curated and pre-processed before analysis. On the other hand, the temporal structures are quite different. Specifically, B flares tend to have longer range dependencies, as the frames leading to these types flares are smooth, which is consistent with results from the previous section.

\begin{figure}[tbh!] 
\centering
\begin{subfigure}[t]{0.47\linewidth}
\centering
\includegraphics[width=\textwidth]{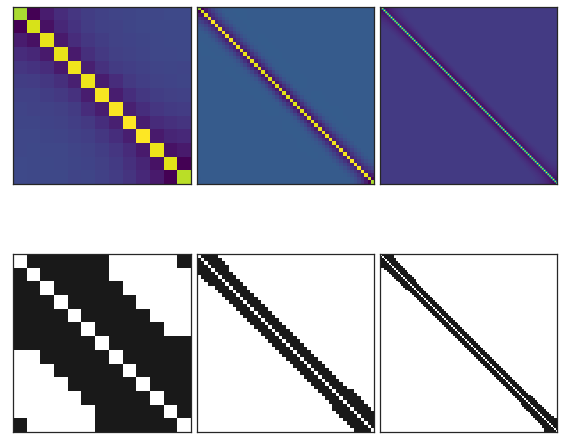}
\caption{Estimated precision matrices - B flares}
\end{subfigure}
\begin{subfigure}[t]{0.47\linewidth}
\centering
\includegraphics[width=\textwidth]{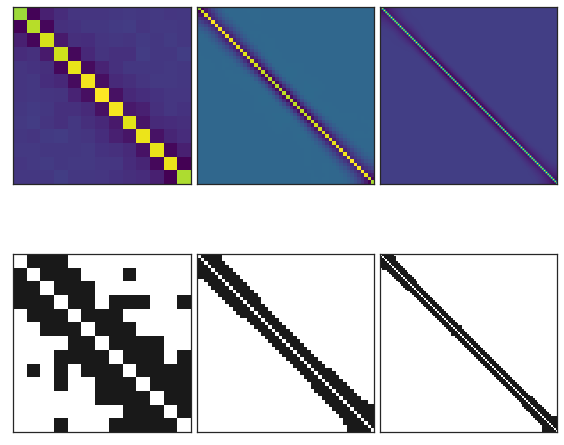}
\caption{Estimated precision matrices - M/X flares}
\end{subfigure}
\caption{Estimated spatial and two (longitude and latitude) temporal Sylvester generating factors for B and MX solar flares, along with their off-diagonal sparsity patterns (second row in each subplot). Both classes exhibit autoregressive dependence structures (across time or space). Note the significant difference in the temporal components, where the B flares exhibit longer range dependency. This is consistent with the smooth transition property of the corresponding videos as illustrated previously.}
\label{fig:psih_all}
\end{figure}

\end{document}